\documentclass{article} 
\usepackage{iclr2027_conference,times}

\usepackage{amsmath,amsfonts,bm}

\def\eqref#1{equation~\ref{#1}}

\def\1{\bm{1}}

\DeclareMathAlphabet{\mathsfit}{\encodingdefault}{\sfdefault}{m}{sl}
\SetMathAlphabet{\mathsfit}{bold}{\encodingdefault}{\sfdefault}{bx}{n}

\newcommand{\E}{\mathbb{E}}

\usepackage{hyperref}
\usepackage{url}
\usepackage{amsthm}
\usepackage{algorithm}
\usepackage{algpseudocode}
\usepackage{amssymb} 
\usepackage{graphicx}
\usepackage{xcolor}
\usepackage{cleveref}
\usepackage{hyperref}

\algrenewcommand\algorithmicrequire{\textbf{Input:}}
\algrenewcommand\algorithmicensure{\textbf{Output:}}

\newtheorem{theorem}{Theorem}
\newtheorem{lemma}{Lemma}

\newtheorem{definition}{Definition}
\newtheorem{remark}{Remark}

\title{Near-Optimal Reinforcement Learning with Multi-Step Transition Lookahead}

\author{Corentin Pla \\
CREST, ENSAE \\
Criteo AI Lab \\
FairPlay Joint Team \\
\texttt{c.pla@criteo.com} \\
\And
Hugo Richard \\
Criteo AI Lab \\
FairPlay Joint Team \\
\texttt{h.richard@criteo.com} \\
\And
Marc Abeille \\
Criteo AI Lab \\
FairPlay Joint Team \\
\texttt{m.abeille@criteo.com} \\
\And
Vianney Perchet \\CREST, ENSAE \\
Criteo AI Lab \\
FairPlay Joint Team \\
\texttt{v.perchet@criteo.com} \\
}

\newcommand{\midsum}{%
  \mathop{\vcenter{\hbox{\scalebox{0.7}{$\displaystyle\sum$}}}}\limits
}

\iclrfinalcopy 
\begin{document}

\maketitle

\begin{abstract}
We study reinforcement learning (RL) with transition look-ahead, where the agent may observe which states would be visited upon playing any sequence of $\ell$ actions before deciding its course of action. Although look-ahead can substantially improve achievable performance, \citet{pla2026on} showed that optimal planning with multi-step transition look-ahead is NP-hard. However, this hardness was established using a discount factor close to one. It was therefore unknown whether the problem remains hard for every discount factor, and whether near-optimal planning can nevertheless be performed efficiently. We resolve both questions. First, we show that for every fixed discount factor, exact planning remains NP-hard. Second, we introduce a randomized polynomial-time approximation scheme for every fixed look-ahead depth. Third, we extend our approach to account for unknown transitions.  We empirically validate the soundness of our results on the wind-farm storage-control benchmark of \citet{lu2025reinforcementlearningimperfecttransition}, showing that our approach, optimally accounting for $\ell$-step look-ahead information, offers substantially better performance than existing algorithms.
\end{abstract}

\section{Introduction}

Reinforcement learning (RL) studies how an agent should act in a dynamic environment to maximize cumulative reward, accounting for both immediate gains and the long-term consequences of its decisions \citep{sutton2018reinforcement}. We consider stationary Markov decision processes (MDPs), whose reward function and transition kernel do not vary over time. RL with look-ahead extends this classical framework by providing the agent with predictive information before it acts. In the transition look-ahead setting, the agent observes, at each decision time, which states would be reached by every action sequence of length at most $\ell$. Such information may arise from real-time forecasts, collaborative navigation systems, or high-fidelity simulators and world models. We refer to \citet{pla2026on,merlis2024reinforcementlearninglookaheadinformation,merlis2024valuerewardlookaheadreinforcement,lu2025reinforcementlearningimperfecttransition} for further motivation and examples. Look-ahead can substantially improve achievable performance, but exploiting it requires policies and planning methods that account explicitly for this richer information structure. Standard RL algorithms do not provide such mechanisms directly, while simply discarding the
additional information can be strictly suboptimal.

\paragraph{Related work.}

The idea of augmenting reinforcement learning with look-ahead information has recently received increasing attention. \citet{merlis2024reinforcementlearninglookaheadinformation} introduced a pseudo-polynomial algorithm for one-step transition look-ahead in the finite-horizon setting, while \citet{merlis2024valuerewardlookaheadreinforcement} studied general look-ahead horizons for \emph{reward} look-ahead, in which rewards are revealed ahead of time, focusing on  worst-case bounds on the ratio between the optimal values with and without look-ahead. Unfortunately,~\citet{pla2026on} ruled out the possibility of leveraging such extra information with reasonable computational cost: for perfect transition look-ahead, planning is polynomial-time solvable for $\ell=1$ but becomes NP-hard for every $\ell\geq 2$. To sidestep this computational barrier,~\citet{lu2025reinforcementlearningimperfecttransition} adopt a batched decision protocol: after observing an $\ell$-step forecast, the agent commits to an entire block of actions before receiving a new forecast. This effectively reduces each block to a single decision and brings the problem back to a one-step look-ahead framework. However, it differs from the original rolling look-ahead setting, where the agent can revise its decisions as the prediction window shifts, and can therefore be strictly suboptimal. Whether near-optimal planning under genuine multi-step transition look-ahead can be performed efficiently thus remains open.

The hardness result also connects to the broader literature on the computational complexity of MDP planning. Classical planning in stationary MDPs is polynomial-time solvable under the discounted criterion, with foundational complexity results due to \citet{papadimitriou1987complexity}; subsequent work has further investigated the complexity of MDP planning under different objectives and horizons \citep{mundhenk2000complexity,littman2013complexitysolvingmarkovdecision,balaji2018complexity,chen2017lowerboundcomputationalcomplexity}. Beyond these classical formulations, changes in the information structure can fundamentally alter computational complexity. For instance, \citet{walsh2009learning} show that planning with delayed feedback can become NP-hard because of the exponential blow-up of the augmented state space, while partial observability leads to even stronger hardness results \citep{papadimitriou1987complexity}. Transition look-ahead exhibits a complementary phenomenon: providing additional information also induces a large augmented state space and makes exact planning intractable (\cite{pla2026on}). This raises a natural algorithmic question: can such planning problems nevertheless be approximated efficiently without explicitly constructing the corresponding augmented MDP?

Two related paradigms provide useful starting points. First, \citet{kearns2002sparse} introduced sparse sampling, which computes near-optimal actions by exploring only a randomly sampled portion of the full look-ahead tree. Its complexity, however, is exponential in the effective horizon $(1-\gamma)^{-1}$. Second, sample-average approximation replaces intractable expectations with a fixed empirical sample and solves the resulting finite optimization problem \citep{kleywegt2002sample}. Adapting this principle to transition look-ahead raises two challenges that have not been addressed: obtaining polynomial dependence on $(1-\gamma)^{-1}$ for fixed look-ahead depth $\ell$, and ensuring that the precomputed solution remains accurate for every look-ahead window that may be observed at deployment.

Transition look-ahead can be viewed as prediction revealed before the agent acts. This connects our setting to the growing literature on algorithms with predictions. This perspective has recently been explored in several sequential decision-making settings. \citet{NEURIPS2024_9dff3b83} design a learning-augmented controller for LQR with latent perturbations, where accurate predictions lead to near-optimal performance while robustness to prediction errors is retained. \citet{lyu2025efficientlysolvingdiscountedmdps} study discounted MDPs equipped with predictions of the transition matrix and show that such predictions can reduce sample complexity, while \citet{li2023blackboxadvicelearningaugmentedalgorithms} establish consistency--robustness trade-offs when the advice is provided in the form of predicted $Q$-values in non-stationary MDPs. Another line of work studies MDPs with exogenous information or dynamics. \citet{pla2026minimaxpacboundslearning} consider discounted MDPs with i.i.d.\ exogenous contexts that are revealed before the agent acts and derive minimax PAC guarantees that exploit this structure, while \citet{maran2026learningmarkovdecisionprocesses} study MDPs with markovian exogenous contexts and show that this structure can substantially improve learning guarantees.

\paragraph{Contribution.}

First, we refine the results of~\citet{pla2026on} by showing that the known NP-hardness does not rely on a large effective horizon: for every fixed rational discount factor $\gamma\in(0,1)$, exact planning remains NP-hard for $\ell\geq 2$ (Theorem~\ref{thm:fixed-gamma-hardness}). Second, for every fixed look-ahead depth $\ell$, we give a randomized polynomial-time approximation scheme (RPTAS) that constructs, with high probability, a uniformly near-optimal policy (Theorem \ref{thm:main-RPTAS}). We then extend our approach to unknown transitions with a regret learning algorithm (Theorem \ref{thm:dolar-regret}). Finally, on the wind-farm storage-control benchmark of \citet{lu2025reinforcementlearningimperfecttransition}, our planning algorithm achieves substantially better performance than their planning algorithm.

\section{Setting and objectives}
\label{sec:perfect-lookahead-setting}
\subsection{Markov Decision Processes }
We study finite tabular Markov decision processes (MDPs) $\mathcal{M}=(\mathcal{S},\mathcal{A},P,r), $ where $\mathcal{S}$ is a finite state space, $\mathcal{A}$ is a finite action space,  $P(s'\mid s,a)$ denotes the probability of reaching state $s' \in \mathcal{S}$ after taking action $a \in \mathcal{A}$ in state $s \in \mathcal{S}$, and $r:\mathcal{S}\times\mathcal{A}\to[0,1]$ is the reward function. A (possibly randomized) stationary memoryless policy is a mapping $\pi:\mathcal{S}\to\Delta(\mathcal{A}),$ where $\Delta(\mathcal{A})$ denotes the simplex over $\mathcal{A}$. Under such a policy, the interaction evolves as follows: at each time $t\in\mathbb{N}$, the system is in state $s_t\in\mathcal{S}$, the agent selects an action $a_t\sim \pi(\cdot\mid s_t)$, receives reward $r(s_t,a_t)$, and the next state is sampled according to $s_{t+1}\sim P(\cdot\mid s_t,a_t).$

We consider the standard discounted-return objective. For a discount factor $\gamma\in(0,1)$ and an initial state $s\in\mathcal{S}$, the value of a policy $\pi$ is $V^\pi(s) = \mathbb{E}^\pi\!\left[ \sum_{t=0}^{\infty}\gamma^t r(s_t,a_t) \,\middle|\, s_0=s \right].$ The optimal discounted value function is defined by $V^*(s) = \sup_{\pi} V^\pi(s),\  \forall s\in\mathcal{S},$ where the supremum is taken over stationary memoryless policies. The optimal value function $V^*$ is the unique solution to the Bellman optimality equations: $V^*(s) = \max_{a\in\mathcal{A}} \Bigl\{ r(s,a)+\gamma\sum_{s'\in\mathcal{S}}P(s'\mid s,a)\,V^*(s') \Bigr\},  \forall s \in \mathcal{S}.$ 

\subsection{Transition look-ahead}
We now formalize the extra information provided by the look-ahead in terms of state observability and provide an augmented MDP construction that allows us to embed this problem into the standard evaluation framework introduced above. In the remainder of this paper, we focus on the perfect look-ahead case.  The extension to noisy look-ahead is given in~\Cref{app:noisy-lookahead}.

\subsubsection{Look-ahead and state observability}
A convenient generative view of the transition kernel is the following. At each time $t$, for every state-action pair $(s,a)$, the environment independently draws a potential successor $\Theta_t(s,a)\sim P(\cdot\mid s,a)$, collecting these draws defines a random transition table $\Theta_t:\mathcal S\times\mathcal A\to\mathcal S,$ taking values in
$\Omega:=\mathcal S^{\mathcal S\times\mathcal A}$. Its distribution $ L$ is therefore
\begin{equation}
L(\theta) := \hspace{1mm} \mathbb{P}\left(\Theta_t= \theta \right) =\!\!\!\!\!\prod_{(s,a)\in\mathcal S\times\mathcal A} P\bigl(\theta(s,a) \mid s,a\bigr).
\label{eq:Q}
\end{equation}
In a standard MDP, $\Theta_t$ is hidden from the agent and, after action $A_t$ is selected, the realized transition is simply $S_{t+1}=\Theta_t(S_t,A_t).$ Transition look-ahead changes the information available before acting. At time $t$, the environment reveals some of the tables $\Theta_t,\Theta_{t+1},\ldots$ to the agent.

\begin{definition}[$\ell$-step transition look-ahead]
\label{def:perfect-transition-lookahead}
Let $\ell\geq1$. At each decision time $t$, an agent with $\ell$-step transition look-ahead observes, before choosing $A_t$, the window 
\begin{equation}
C_t^\ell
:=
(\Theta_t,\Theta_{t+1},\ldots,\Theta_{t+\ell-1})
\in\Omega^\ell.
\label{eq:perfect-window}
\end{equation}
\end{definition}

This window encodes the complete depth-$\ell$ transition tree rooted at the current state. Indeed, starting from $s_0=S_t$, any action sequence $a_0,\ldots,a_{k-1}$, with $k\leq\ell$, determines the trajectory
\begin{equation*}
s_{j+1} = \Theta_{t+j}(s_j,a_j), \qquad j=0,\ldots,k-1.
\end{equation*}
Thus, $C_t^\ell$ reveals the outcome of every action sequence of length at most $\ell$ before the first action is chosen. In particular, for $\ell=1$, the agent knows the successor $\Theta_t(S_t,a)$ of every action $a\in\mathcal A$. After one transition, the oldest table is discarded and a fresh independent table is revealed, so that $C^\ell_{t+1} = (\Theta_{t+1},\ldots,\Theta_{t+\ell-1},\Theta_{t+\ell}),$ with $\Theta_{t+\ell}\sim L$ independent of the previous tables. Hence, the pair $(S_t,C_t^\ell)$ evolves as a Markov process on $\mathcal S\times\Omega^\ell$, allowing the transition look-ahead problem to be treated as a standard discounted MDP on this augmented state space.

\subsubsection{Policies and value functions}

A stationary policy with $\ell$-step transition look-ahead acts on the augmented state and is therefore a mapping $\pi:\mathcal S\times\Omega^\ell\to\Delta(\mathcal A).$ Its value at an observed state-window pair $(s,c) \in \mathcal{S}\times \Omega^\ell $, with $c=(\theta_0,\ldots,\theta_{\ell-1})$, is $V_\ell^\pi(s,c) := \mathbb E^\pi\left[ \sum_{t=0}^{\infty}\gamma^t r(S_t,A_t) \,\middle|\, S_0=s,\ C_0^\ell=c \right].$

We denote the optimal value by $V_\ell^\star(s,c) := \sup_\pi V_\ell^\pi(s,c).$ For any bounded function $V:\mathcal S\times\Omega^\ell\to\mathbb R$, define the Bellman optimality operator $(\mathcal T_\ell V)(s,\theta_0,\ldots,\theta_{\ell-1}) := \max_{a\in\mathcal A} \left\{ r(s,a) + \gamma \mathbb E_{\Theta\sim L} \left[ V\bigl( \theta_0(s,a), \theta_1,\ldots,\theta_{\ell-1}, \Theta \bigr) \right] \right\}.$ The operator $\mathcal T_\ell$ is a $\gamma$-contraction, so $V_\ell^\star$ is its unique fixed point. The corresponding optimal action-value function is $Q_\ell^\star(s,\theta_0,\ldots,\theta_{\ell-1},a) := r(s,a) + \gamma \mathbb E_{\Theta\sim L} \left[ V_\ell^\star\bigl( \theta_0(s,a), \theta_1,\ldots,\theta_{\ell-1}, \Theta \bigr) \right].$ An optimal policy is then obtained by choosing $\pi_\ell^\star(s,\theta_0,\ldots,\theta_{\ell-1}) \in \arg\max_{a\in\mathcal A} Q_\ell^\star(s,\theta_0,\ldots,\theta_{\ell-1},a).$
\section{Hardness of exact planning}

Our first result strengthens the known hardness of multi-step transition look-ahead planning of \cite{pla2026on} by showing that it persists for every fixed discount factor.
\begin{theorem}[Fixed-discount hardness of transition look-ahead planning]
\label{thm:fixed-gamma-hardness}
Fix an integer $\ell\geq 2$ and a rational discount factor
$\gamma\in(0,1)$. Given a finite MDP $\mathcal M$, an initial state
$s_0 \in \mathcal{S}$, and a rational threshold $\theta$, deciding whether there exists a
policy with perfect $\ell$-step transition look-ahead such that $\mathbb E_{C\sim L^{\otimes\ell}} \bigl[V_\ell^{\pi}(s_0,C)\bigr] \geq\theta $ is NP-hard.
\end{theorem}

\begin{proof}[Proof sketch]
We sketch the proof for $\ell = 2$ and refer the reader to \Cref{proof:fixed-gamma-hardness} for a formal derivation.  The proof is almost identical to that of~\cite{pla2026on}. We reduce from the following expected-maximum subset problem of \citet{mehta2020hittinghighnotessubset}. Given independent, nonnegative,
finite-support random variables $X_1,\ldots,X_n$, an integer $k$ and a threshold $U \in \mathbb{Q}$, deciding whether there exists a set $S \subset [n]$ such that 
$|S| = k$ and $\E[\max_{i \in S} X_i] > U$ is NP-hard.

Given the input $(X_i)_{i \in [n]}, k, U$ we show that there exists an MDP with look-ahead $\ell = 2$ of size polynomial in the size of the input (in particular in $n$ and the support size)  such that finding the optimal policy is equivalent to finding the optimal set $S^*$ that maximizes  $\E[\max_{i \in S} X_i]$ over the set $\{S: S \subset [n], |S| = k \}$. The MDP is the following. At the root state $s_0$, the agent can either choose $\mathsf{wait}$ and stay at
$s_0$, or choose $\mathsf{go}$ and go to $s_1$. At $s_1$, there are $k$ actions
$\mathsf{pick}_1,\ldots,\mathsf{pick}_k$, each transitioning independently and
uniformly to one of the states $s^2_1,\ldots,s^2_n$ associated with the random
variables. Consequently, the look-ahead window available at $s_0$ reveals a
candidate set of states  $\Sigma=\{q_1,\ldots,q_k\} \subset \{s^2_1, \dots, s^2_n\}$ before the agent decides
whether to leave $s_0$. Choosing $\mathsf{wait}$ reveals a fresh independent
tuple while staying at $s_0$. If the agent chooses $\mathsf{go}$, then at $s_1$
the look-ahead reveals the payoffs associated with $\Sigma$, allowing it to select
the largest one. 
The rest of the MDP is chosen such that the payoffs associated with  $s^2_1, \dots, s^2_n$ have the same 
distribution as $X_1, \dots, X_n$. If the optimal strategy is to wait until the best subset appears, then 
solving the MDP is the same as solving the expected-maximum subset problem.

However, if waiting does not give any reward and if the discount is low enough, it may be more rewarding to stop before seeing the optimal set.
This is why the NP-hardness result in~\cite{pla2026on} requires the discount factor to be exponentially close to one so that the cost of waiting is negligible.Instead, in our proof, we set $T:=\gamma^{\ell+1}U$ and assign
the $\mathsf{wait}$ action a reward of $(1-\gamma)T$, so that
waiting exactly preserves the threshold value $T$:
$(1-\gamma)T+\gamma T=T$. The rest of the proof follows~\cite{pla2026on}.

\end{proof}

\section{Near-optimal planning}
\label{subsec:rptas-offline-online-complexity}
\subsection{Algorithm}

Despite this hardness of exact planning, near-optimal planning
remains tractable: for every fixed look-ahead depth, we provide a randomized
polynomial-time approximation scheme (RPTAS). Our algorithm proceeds in two
phases. Offline, it samples a finite collection of transition tables and uses
them to build a finite state space. Value iteration is then run on this
sampled state space. Online, the agent combines the precomputed values with the
actual look-ahead window it observes in order to select an action. 

\paragraph{Offline phase.}
The offline phase constructs a finite proxy MDP by replacing the distribution
$L$ of transition tables with an empirical distribution. Given a dictionary size
$N$, sample  $\Theta^1,\ldots,\Theta^N\overset{\mathrm{i.i.d.}}{\sim}L$ and define $
\widehat L_N:=\frac1N\sum_{j=1}^N\delta_{\Theta^j}.$ Under $\widehat L_N$, every
look-ahead window consists of tables from the dictionary and can therefore be
identified with their indices. The proxy MDP consequently has the finite state
space $\mathcal{D}_N:=\mathcal S\times[N]^\ell$ where $(s,i_0,\ldots,i_{\ell-1})$
represents the augmented state $(s,\Theta^{i_0},\ldots,\Theta^{i_{\ell-1}})$.
From this state, action $a$ leads to
$\bigl(\Theta^{i_0}(s,a),i_1,\ldots,i_{\ell-1},J\bigr), \ \
J\sim\operatorname{Unif}([N]).$ Thus, the current table determines the next
state, the remaining tables are shifted forward, and the new table is sampled
uniformly from the dictionary.

Algorithm~\ref{alg:rptas-preprocessing} approximately solves this proxy MDP by
running $K$ value-iteration steps on $\mathcal{D}_N$, starting from $V^0=0$, and
returns $V^K:\mathcal{D}_N\to[0,V_{\max}], \ V_{\max}=\frac1{1-\gamma}.$ If
$V_N^\star$ denotes the optimal value function of the proxy MDP, contraction of
its Bellman operator gives $\lVert V^K-V_N^\star\rVert_\infty \leq \gamma^K
V_{\max}.$ Hence, $K\geq \log(1/\eta)/(-\log\gamma)$ iterations suffice to
achieve error at most $\eta V_{\max}$. Each iteration takes $O(|\mathcal{S}||\mathcal{A}|N^{\ell+1})$
operations, and the resulting offline representation
$D=(\Theta^1,\ldots,\Theta^N,V^K)$ is computed once, without enumerating the
original space $\Omega^\ell$.

\paragraph{Online phase.}
At decision time, the agent observes its current state $s \in \mathcal{S}$ and an arbitrary look-ahead window
$c=(\theta_0,\ldots,\theta_{\ell-1})\in\Omega^\ell$. Since these tables need not belong to the sampled dictionary, the proxy value $V^K$ cannot be evaluated directly at $(s,c)$. The key observation is that, after $\ell$ transitions, every table in the current window has been shifted out and replaced by fresh tables, which are represented using the empirical distribution $\widehat L_N$. At that point, the resulting window belongs to the proxy state space $\mathcal D_N$, and $V^K$ provides an appropriate terminal value.

Algorithm~\ref{alg:rptas-query} therefore solves an $\ell$-step planning problem with terminal value $V^K$. This problem is solved by backward induction because $V^K$ is available only at the end of the horizon. For $m\in\{1,\ldots,\ell\}, u \in \mathcal{S}, i_0, \ldots, i_{m-1} \in [N]$, the quantity $U^{V^K}_{m,c}(u,i_0,\ldots,i_{m-1})$ represents the continuation value at step $m$, when the remaining tables from the observed window are $\theta_m,\ldots,\theta_{\ell-1}$ and the first $m$ newly revealed tables are represented by $i_0,\ldots,i_{m-1}$. Formally, $U^{V^K}_{m,c}$ is defined recursively for $m\in\{1,\ldots,\ell\}$. The case $m = \ell$ is given by 
\begin{equation}
U^{V^K}_{\ell,c}(u,i_0,\ldots,i_{\ell-1}):= V^K(u,i_0,\ldots,i_{\ell-1}),
\end{equation}
for any $u \in \mathcal{S}, i_0, \ldots, i_{\ell-1} \in [N]$.

Then, for $m\in \{1, \ldots, \ell-1\}, u \in \mathcal{S}, i_0, \ldots, i_{m-1} \in [N]$, we define
\begin{equation}
U^{V^K}_{m,c}(u,i_0,\ldots,i_{m-1}) := \max_{a\in\mathcal A} \left\{ r(u,a) + \frac{\gamma}{N} \sum_{i_m=1}^N U^{V^K}_{m+1,c} \bigl(\theta_m(u,a),i_0,\ldots,i_m\bigr) \right\}.
\label{eq:uk}
\end{equation}
Here, $\theta_m$ determines the next physical state, while the newly entering table is drawn uniformly from the dictionary. Importantly, each maximization is performed after conditioning on the tables available at that stage, so the recursion accounts for future adaptation rather than committing in advance to an open-loop action sequence. Finally, the score at state $s \in \mathcal{S}$, look-ahead window $c = (\theta_0, \ldots, \theta_{\ell-1}) \in \Omega^\ell$ and action $a\in\mathcal{A}$ is defined by
\begin{equation}
\widetilde Q_{V^K}(s,c,a) := r(s,a) + \frac{\gamma}{N} \sum_{i_0=1}^N U^{V^K}_{1,c}\bigl(\theta_0(s,a),i_0\bigr),
\end{equation}
and the policy selects $\pi_D(s,c)\in \arg\max_{a\in\mathcal A} \widetilde Q_{V^K}(s,c,a).$
Only this first action is executed; at the next decision time, the procedure is repeated using the shifted look-ahead window.

\begin{algorithm}[H]
\caption{RPTAS-offline}
\label{alg:rptas-preprocessing}
\begin{algorithmic}[1]
\Require MDP $\mathcal M=(\mathcal S,\mathcal A,P,r)$, look-ahead depth
$\ell$, dictionary size $N$, iterations $K$
\Ensure Policy representation $D$
\State Sample $\Theta^1,\ldots,\Theta^N
\overset{\mathrm{i.i.d.}}{\sim}L$
\State Set $\mathcal D_N\gets\mathcal S\times[N]^\ell$
\State Initialize $V^0(s,i_0,\ldots,i_{\ell-1})\gets0$
for all $(s,i_0,\ldots,i_{\ell-1})\in\mathcal D_N$
\For{$k=0,\ldots,K-1$}
    \ForAll{$(s,i_0,\ldots,i_{\ell-1})\in\mathcal D_N$}
        \State \refstepcounter{equation}\label{eq:vi-finite-statespace}%
        $\displaystyle V^{k+1}(s,i_0,\ldots,i_{\ell-1})\gets\max_{a\in\mathcal A} \{ r(s,a)+\frac{\gamma}{N}\midsum_{j=1}^N   V^k\bigl(    \Theta^{i_0}(s,a),i_1,\ldots,i_{\ell-1},j   \bigr)    \}$
        \hfill\textup{(\theequation)}
    \EndFor
\EndFor
\State $D\gets(\Theta^1,\ldots,\Theta^N,V^K)$
\State \Return $D$
\end{algorithmic}
\end{algorithm}
\vspace{-\intextsep}
\begin{algorithm}[H]
\caption{RPTAS-online}
\label{alg:rptas-query}
\begin{algorithmic}[1]
\Require $D=(\Theta^1,\ldots,\Theta^N,V^K)$, current state $s$, observed window $c=(\theta_0,\ldots,\theta_{\ell-1})$
\Ensure Action $\pi_D(s,c)$
\State
$U^{V^K}_{\ell,c}(u,i_0,\ldots,i_{\ell-1}) \gets V^K(u,i_0,\ldots,i_{\ell-1})$
for all $(u,i_0,\ldots,i_{\ell-1})\in\mathcal D_N$

\For{$m=\ell-1,\ldots,1$}
    \ForAll{$(u,i_0,\ldots,i_{m-1})
    \in\mathcal S\times[N]^m$}
        \State
$U^{V^K}_{m,c}(u,i_0,\ldots,i_{m-1})
\gets
\displaystyle\max_{a\in\mathcal A}
\{
r(u,a)+\frac{\gamma}{N}
{\midsum_{i_m=1}^N}
U^{V^K}_{m+1,c}\bigl(\theta_m(u,a),i_0,\ldots,i_m\bigr)
\}$
    \EndFor
\EndFor

\ForAll{$a\in\mathcal A$}
    \State
    $\widetilde Q_{V^K}(s,c,a)\gets\displaystyle r(s,a)+\frac{\gamma}{N}\midsum_{i_0=1}^N U^{V^K}_{1,c}\bigl( \theta_0(s,a),i_0 \bigr)$
\EndFor

\State \Return
$\displaystyle
\pi_D(s,c)
\gets
\arg\max_{a\in\mathcal A}\widetilde Q_{V^K}(s,c,a)$
\end{algorithmic}
\end{algorithm}

\subsection{Theoretical Guarantees}
\label{subsec:rptas-analysis}
\begin{theorem}[RPTAS with perfect $\ell$-step look-ahead]
\label{thm:main-RPTAS}
Fix $\ell\geq 2$ and $\varepsilon,\delta\in(0,1)$. For $N=\widetilde O(\frac{\ell |\mathcal{S}| |\mathcal{A}|}{\varepsilon^2(1-\gamma)^4})$ and $K=\widetilde O(\frac{1}{1-\gamma })$, \Cref{alg:rptas-preprocessing,alg:rptas-query} compute a policy $\pi$ such that
\begin{equation}
\mathbb{P}\Bigl( V_\ell^\pi(s,c) \geq V_\ell^\star(s,c)-\varepsilon V_{\max}, \qquad \forall (s,c)\in\mathcal S\times\Omega^\ell\Bigr) \geq 1-\delta.
\label{eq:main-policy-guarantee}
\end{equation}
For every fixed $\ell$, computing and storing the policy, as well as selecting an action at each decision time, require time and memory polynomial in
$|\mathcal{S}|$, $|\mathcal{A}|$, $\varepsilon^{-1}$, $\log(1/\delta)$, and $(1-\gamma)^{-1}$.
\end{theorem}

\begin{proof}[Proof sketch for $\ell=2$]
We separate the proof into three steps: (i) approximating the distribution of
transition tables, (ii) solving the resulting empirical MDP on dictionary windows,
and (iii) extending this solution to an arbitrary observed window.

(i) Let $\widehat L_N=N^{-1}\sum_{j=1}^N\delta_{\Theta^j}$, and let
$\widehat V^\star$ and $\widehat Q^\star$ be the optimal value and
action-value functions when $L$ is replaced by $\widehat L_N$. In particular $\widehat V^\star$ is the fixed point of $\widehat{\mathcal T}_2$ defined by
\begin{equation}
(\widehat{\mathcal T}_2 V)(s,\theta_0,\theta_1)
=\max_{a\in\mathcal A}\left\{
r(s,a)+\frac{\gamma}{N}\sum_{j=1}^N
V\bigl(\theta_0(s,a),\theta_1,\Theta^j\bigr)
\right\}.
\end{equation}
Set $\eta:=\varepsilon(1-\gamma)^2V_{\max}/4$. The functions
$\Theta\mapsto V_2^\star(u,\theta_1,\Theta)$ take values in
$[0,V_{\max}]$, and there are at most $|\mathcal{S}||\Omega|=|\mathcal{S}|^{1+|\mathcal{S}||\mathcal{A}|}$ such functions.
Hence, if $N\geq \frac{V_{\max}^2}{2\eta^2}
\left(\log\frac{2}{\delta}+(1+|\mathcal{S}||\mathcal{A}|)\log |\mathcal{S}|\right),$
Hoeffding's inequality and a union bound give, with probability at
least $1-\delta$, $\sup_{u,\theta_1}\left|
\frac1N\sum_{j=1}^N V_2^\star(u,\theta_1,\Theta^j)
-\mathbb E_{\Theta\sim L}
  [V_2^\star(u,\theta_1,\Theta)]
\right|\leq\eta.$ On this event, contraction of the Bellman operators yields $\|\widehat V^\star-V_2^\star\|_\infty
\leq\frac{\gamma\eta}{1-\gamma},
\
\|\widehat Q^\star-Q_2^\star\|_\infty
\leq\frac{\gamma\eta}{1-\gamma}.$

(ii) Now let us apply the Bellman optimality equation in the proxy MDP whose state space is $\mathcal D_N=\mathcal S\times[N]^2$, where $(s,i_0,i_1)$ represents $(s,\Theta^{i_0},\Theta^{i_1})$.
\begin{align*}
    V^*_N(s,i_0,i_1)&=\max_{a\in\mathcal A}\left\{
r(s,a)+\frac{\gamma}{N}\sum_{j=1}^N
\widehat V^\star\bigl(
\Theta^{i_0}(s,a),\Theta^{i_1},\Theta^j
\bigr)\right\}\\& =\widehat{\mathcal{T}}_2\widehat V^\star(s,\Theta^{i_0},\Theta^{i_1})=\widehat V^\star(s,\Theta^{i_0},\Theta^{i_1})
\end{align*}

Thus, value iteration on $\mathcal D_N$ computes the empirical optimal value on every dictionary window. After $K$ iterations, $\|V^K-V_N^\star\|_\infty\leq\gamma^K V_{\max}. $

(iii) Let $c=(\theta_0,\theta_1)\in\Omega^2$. Unrolling the empirical Bellman
equation gives
\begin{align}
\widehat V^\star(s,\theta_0,\theta_1) = \max_{a_0\in\mathcal A} \Bigg\{ &r(s,a_0) +\frac{\gamma}{N}\sum_{i_0=1}^N \max_{a_1\in\mathcal A} \Bigg[ r\bigl(\theta_0(s,a_0),a_1\bigr) \nonumber\\ &\quad+ \frac{\gamma}{N}\sum_{i_1=1}^N \underbrace{\hat V^\star\Bigl( \theta_1\bigl(\theta_0(s,a_0),a_1\bigr), i_0,i_1 \Bigr)}_{(II)} \Bigg] \Bigg\}.
\label{eq:forward-extension-l2}
\end{align}
In particular, $(II)$  contains only dictionary tables. Hence, as proven at step 2, $ (II) = V_N^\star\bigl(\theta_1(s_1,a_1),i_0,i_1\bigr).$ Therefore, for a fixed root action $a_0$, the expression inside the outer maximum is exactly the score computed by the online recursion. Thus, $\widetilde Q_{V_N^\star}(s,c,a_0) = \widehat Q^\star(s,c,a_0).$ See Appendix \ref{proof:main-RPTAS} for the full proof and Appendix \ref{app:noisy-planning} for its extension to noisy look-ahead. 
\end{proof}

\section{Regret minimization}
\label{sec:regret-minimization}
We now consider the online setting in which the transition kernel $P$ is initially unknown. Note that each revealed transition table provides one sample from every row $P(\cdot\mid s,a)$, regardless of the chosen action. Learning $P$ therefore requires no exploration: we can repeatedly estimate the model and apply our RPTAS to it.  Following the $\gamma$-regret of \citet{liu2021regretboundsdiscountedmdps}, we apply this criterion to the augmented state $X_t=(S_t,C_t)$:
\begin{equation}
\operatorname{Reg}(T)
:=\sum_{t=0}^{T-1}\bigl[(1-\gamma)V_\ell^\star(X_t)-r_t\bigr],
\qquad r_t=r(S_t,A_t).
\label{eq:online-regret}
\end{equation}
The factor $1-\gamma$ puts the optimal discounted value on the scale of a one-step reward, so the regret compares this benchmark at each encountered state with the rewards actually collected. To control this regret, \textsc{RPTAS-Learn} periodically rebuilds the dictionary and value array used for planning, using an empirical transition model. Updates occur at times $M\in\{1,2,4,\ldots\}$. At each update, it uses all tables preceding the current window to form
\begin{equation}
\widehat P_M(s'\mid s,a) :=\frac1M\sum_{j=0}^{M-1} \mathbf 1\{\Theta_j(s,a)=s'\}, \qquad \widehat L_M :=\bigotimes_{s,a}\widehat P_M(\cdot\mid s,a).
\label{eq:online-product-model}
\end{equation}
This includes every state--action row, whether visited or not. The $\ell$ tables in the current window are used directly for action selection. We then run Algorithm~\ref{alg:rptas-preprocessing} on $\widehat{\mathcal M}=(\mathcal S,\mathcal A,\widehat P_M,r)$: draw $N_M$ independent tables from $\widehat L_M$ and perform $K_M$ ordinary value-iteration steps from zero. The resulting dictionary and value array are held fixed until the next update. At each decision, Algorithm~\ref{alg:rptas-query} combines them with the current window to select an action. Both planning routines are unchanged. Updating only at powers of two requires $O(\log(T+1))$ planner recomputations while keeping the active sample size within a factor of two of the elapsed time.
\begin{algorithm}[H]
\caption{\textsc{RPTAS-Learn}}
\label{alg:online-perfect-lookahead}
\begin{algorithmic}[1]
\Require $T$, $\gamma$, $\ell$, $r$, parameters $(N_M,K_M)$
\State Observe $X_0$ and execute a fixed action $A_0\in\mathcal A$
\For{$t=1,\ldots,T-1$}
    \State Observe $X_t=(S_t,C_t^\ell)$
    \If{$t\in\{2^q:q\in\mathbb N_0\}$}
        \State $M\gets t$
        \State Compute $\widehat P_M$ using
        \eqref{eq:online-product-model}
        \State $D\gets
        \Call{RPTAS-offline}{\widehat{\mathcal M},\ell,N_M,K_M}$
    \EndIf
    \State $A_t\gets\Call{RPTAS-online}{D,S_t,C_t^\ell}$
    \State Execute $A_t$ and receive $r_t$
\EndFor
\end{algorithmic}
\end{algorithm}
We use a fixed rule to break ties, so each planner defines a deterministic stationary policy on the augmented state space.

\begin{theorem}[Regret of \textsc{RPTAS-Learn}]
\label{thm:dolar-regret}
Fix $\ell\geq1$. At every update, choose $N_M=\left\lceil\frac{M}{(1-\gamma)^3}\right\rceil$ and $K_M=\left\lceil \frac1{1-\gamma}\log\frac{M^2}{1-\gamma} \right\rceil.$ For every $T\geq1$ and $\delta\in(0,1)$, with probability at least $1-\delta$,
\begin{equation}
\operatorname{Reg}(T) \leq\widetilde O_\ell\left( \sqrt{\frac{|\mathcal{S}||\mathcal{A}|T}{1-\gamma}} +\frac{|\mathcal{S}|^2|\mathcal{A}|}{(1-\gamma)^2} \right).
\end{equation}
\label{thm:learning}
\end{theorem}

Here $\widetilde O_\ell$ hides logarithmic factors in $|\mathcal{S}|,|\mathcal{A}|,T,1/\delta$ and factors depending on $\ell$. An update based on $M$ tables costs $O\bigl(M|\mathcal{S}||\mathcal{A}|+N_M|\mathcal{S}||\mathcal{A}|+K_M|\mathcal{S}|(|\mathcal{A}|+1)N_M^\ell\bigr)$, and each subsequent action query costs $O\bigl(\ell |\mathcal{S}|(|\mathcal{A}|+1)N_M^\ell\bigr)$. Thus, for fixed $\ell$, both costs are polynomial in $|\mathcal{S}|,|\mathcal{A}|,T$ and $(1-\gamma)^{-1}$.

\begin{proof}[Proof sketch]
At each deterministic update time $M$, the observed tables provide $M$ independent samples from every state--action row, exactly as under uniform sampling from a generative model. We therefore buid upon  the empirical-model comparison and absorption argument of \citet{azar2012samplecomplexityreinforcementlearning}, combining Bernstein concentration with a discounted variance bound. Applying this strategy to transition look-ahead requires controlling the product empirical law and the dictionary approximation; obtaining the $\gamma$-regret of \citet{liu2021regretboundsdiscountedmdps} additionally requires relating policy values to realised rewards.

The first difficulty is that $\widehat L_M$ is a product of empirical row distributions, rather than the empirical distribution of complete tables. By independently rematching the samples within each row, we obtain a Bernstein bound for its error on any fixed function, with variance under $L$ (Lemma~\ref{lem:lazy-product-bernstein}). Applied to continuations of $V_\ell^\star$, this bound combines with a discounted resolvent variance bound to control the propagation of model errors under each deployed policy (Lemma~\ref{lem:lazy-resolvent}). To handle the dependence of $\widehat V_M^\star$ on the samples, we decompose $\widehat V_M^\star
=
V_\ell^\star+(\widehat V_M^\star-V_\ell^\star).$, control the second part uniformly using a product-law KL bound (Lemma \ref{lem:lazy-confidence}), and absorb it in the Bellman comparison once $M\geq M_0=\widetilde O(|\mathcal S|^2|\mathcal A|/(1-\gamma)^2)$.

Conditionally on the observed tables, the dictionary is freshly
sampled from $\widehat L_M$. The known-model RPTAS analysis therefore controls its planning error, including at windows outside the dictionary. Together, these arguments bound the normalized loss of each deployed policy by $\widetilde O_\ell\!\left( \sqrt{|\mathcal S||\mathcal A|/[M(1-\gamma)]} +|\mathcal S||\mathcal A|/[M(1-\gamma)]+M^{-2} \right)$. Finally, the active policy is fixed between each lazy update. Its Bellman equation makes the value-to-reward differences telescope up to epoch boundaries; a variance-sensitive Freedman bound controls the remaining martingale. Summing the policy losses over epochs and bounding the initial $M<M_0$ steps trivially gives the claim.
The full proof is given in Appendix~\ref{proof:regret}.
\end{proof}

\section{Numerical studies}
We evaluate RPTAS on the wind-farm storage-control benchmark of \citet{lu2025reinforcementlearningimperfecttransition}. The wind farm incurs penalties when production falls short of its scheduled production, while a capacity-constrained battery can store surplus energy or compensate for later shortages. Choosing when to charge or discharge therefore depends critically on forecasts of future prices and production mismatches.

Our theory is presented for perfect look-ahead, whereas the benchmark provides noisy forecasts. As shown in \Cref{app:noisy-planning}, RPTAS extends to this setting by replacing each revealed successor with its posterior distribution conditional on the forecast. We use the same data, discretization, battery dynamics, and evaluation interval as \citet{lu2025reinforcementlearningimperfecttransition}. The MDP has $2{,}100$ states, 9 actions, and $\gamma=0.95$. At date $t$, the current price and mismatch are observed exactly, with predictions for $t+1,\ldots,t+\ell$, where $\ell\in\{0,1,2,3\}$.  Forecasts are perturbed as $\widehat x_u=x_u(1+\sigma\xi_u)$, where $x_u$ denotes either the electricity price or the production
mismatch at date $u$, and $\xi_u\sim\mathcal N(0,1)$ with $\sigma\in\{0,.05,.10,.20,.30,.40,.50,.60,.80,1\}$. Following their convention, noise is reported as $100\sigma\%$. RPTAS converts these forecasts into posterior kernels and uses $N=8$ dictionary samples.

We compare RPTAS with three controllers. \textsc{Bola} is the
zero-terminal, point-forecast controller used in the wind-farm
experiments of \citet{lu2025reinforcementlearningimperfecttransition}.
\textsc{Bola-Bayes} follows the method described in their paper: it
uses forecast posteriors and a Bayesian continuation value computed
offline. Both execute a block of $\ell+1$ actions before replanning.
We also include \textsc{Mpc}, which uses the same finite-horizon
objective and zero terminal value as \textsc{Bola}, but replans after
every transition. Comparing \textsc{Mpc} with \textsc{Bola} isolates
the effect of replanning. Comparing it with RPTAS shows why
replanning alone need not suffice: RPTAS also accounts for the value
beyond the forecast window through its dictionary value function.
All methods face the same realised trajectory, initial battery state,
constraints, and forecast-noise draws.

We report realised cost reduction relative to the no-prediction
controller used by \citet{lu2025reinforcementlearningimperfecttransition}.
The no-prediction controller plans as if the current price and
production mismatch remained fixed; it defines the grey reference
line at zero. The red dashed line instead shows the optimal
zero-look-ahead policy for our estimated MDP ($\ell=0$). Higher cost
reduction is better. Curves average ten common forecast-noise seeds,
with two-standard-error bands, see \Cref{app:wind-experiments} for more details on the implementation.

\begin{figure}[H]
\centering
\setlength{\abovecaptionskip}{0pt}
\includegraphics[width=\textwidth]{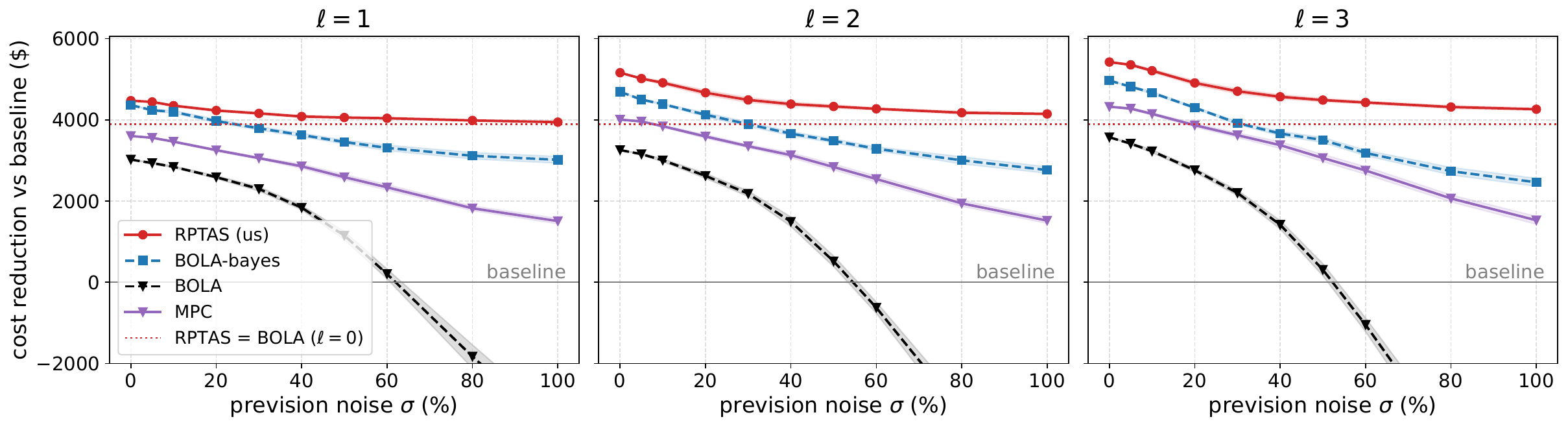}
\vspace{-0.3cm}
\caption{Realised cost reduction relative to the no-prediction baseline versus forecast noise $\sigma$, for $\ell\in\{1,2,3\}$. The grey and red lines denote the no-prediction and classical MDP ($\ell=0$) benchmarks, respectively. Curves show means over $10$ seeds, with $\pm2$ standard-error bands.}
\label{fig:cost-reduction}
\end{figure}

\section{Conclusion}

We investigated whether multi-step transition look-ahead can be used efficiently. Although exact planning is NP-hard for every fixed rational discount factor and every fixed depth $\ell\geq2$, a uniformly near-optimal policy can be computed in polynomial time for fixed $\ell$. Our algorithm samples a polynomial number $N$ of transition tables and runs value iteration on $\mathcal S\times[N]^\ell$, rather than on the exponentially large space $\mathcal S\times\Omega^\ell$. A backward recursion then uses the resulting values to act on \emph{any} observed look-ahead window, yielding a uniformly near-optimal policy with polynomial preprocessing and per-decision cost for fixed $\ell$. With unknown dynamics, the same representation yields a learning algorithm whose leading regret term is comparable to that of classical discounted RL. We empirically validate the soundness of our results on the wind-farm storage-control benchmark of \citet{lu2025reinforcementlearningimperfecttransition}, showing that our approach, optimally accounting for $\ell$-step look-ahead information, offers substantially better performance than existing algorithms. A natural direction is to move beyond the local corruption model and study forecasts whose errors are biased or correlated across state–action pairs and over time, as well as forecasts that reveal only partial information about future transitions. Also, in practice, such predictions may be continuous, high-dimensional, and generated by a learned world model. Extending the dictionary approach to this setting, while relating planning performance and regret to forecast quality, is an important step toward practical forecast-aware reinforcement learning.

\subsection*{AI use statement}
In this work, we used generative AI tools to implement code for our experiments. Additionally, we used generative AI tools to help draft parts of the manuscript. We have reviewed all AI-assisted work: LLM-generated code was verified and tested for correctness by the authors before being used in our experiments, and LLM-assisted text was checked and revised by the authors for accuracy and originality. We take responsibility for the final content of this work, including text, claims or artifacts produced with the aid of
generative AI.

\bibliography{iclr2027_conference}
\bibliographystyle{iclr2027_conference}

\appendix 

\section{Appendix}

\subsection{Proof of Theorem~\ref{thm:fixed-gamma-hardness}}
\label{proof:fixed-gamma-hardness}

We use the notation of Section~\ref{sec:perfect-lookahead-setting}: $\Omega=\mathcal S^{\mathcal S\times\mathcal A}$ is the space of one-step transition tables, $L$ is the product law \eqref{eq:Q} induced by the kernel $P$, and $V_\ell^\star$ denotes the optimal value function on the augmented state space $\mathcal S\times\Omega^\ell$. For a policy $\pi$ and an initial state $s_0\in\mathcal S$, define the values averaged over the initial look-ahead window,
\begin{equation}
v_{\ell,\gamma}^{\pi}(s_0) := \mathbb E_{C\sim L^{\otimes\ell}} \bigl[V_\ell^{\pi}(s_0,C)\bigr], \qquad v_{\ell,\gamma}^{\star}(s_0) := \mathbb E_{C\sim L^{\otimes\ell}} \bigl[V_\ell^{\star}(s_0,C)\bigr],
\label{eq:avg-initial-value}
\end{equation}
For a window $c=(\theta_0,\ldots,\theta_{\ell-1})\in\Omega^\ell$, a state $s\in\mathcal S$, and an action sequence $\sigma=(a_0,\ldots,a_{k-1})$ with $k\leq\ell$, we write $\phi_c(s,\sigma)\in\mathcal S$ for the state reached after playing $\sigma$ from $s$ under the window $c$, that is, the state $s_k$ obtained from the recursion
\begin{equation*}
s_0=s, \qquad s_{j+1}=\theta_j(s_j,a_j), \qquad j=0,\ldots,k-1.
\end{equation*}

\begin{definition}[Discounted Value Decision Problem ($\ell$-DVDP)]
\label{def:l-dvdp}
Given a finite MDP $\mathcal M=(\mathcal S,\mathcal A,P,r)$, a rational discount factor $\gamma\in(0,1)$, an initial state $s_0\in\mathcal S$, and a rational threshold $\theta$, decide whether there exists a policy $\pi$ with $\ell$-step transition look-ahead such that $v_{\ell,\gamma}^{\pi}(s_0)\geq\theta$, equivalently whether $v_{\ell,\gamma}^{\star}(s_0)\geq\theta$.
\end{definition}

Fix an integer $\ell\geq 2$ and a rational discount factor $\gamma\in(0,1)$. We build upon \citet{pla2026on} and exhibit a polynomial-time reduction from \textsc{Independent Set} on regular graphs to $\ell$-\textsc{DVDP} that is valid for this fixed $\gamma$; NP-hardness is therefore preserved even for $\gamma$ far from $1$. We use the following consequence of the reduction of \citet{mehta2020hittinghighnotessubset}; it is the same gap gadget as in the hardness proof of \citet{pla2026on}. The graph parameters are denoted $n_G$ and $m_G$.

\begin{lemma}[Expected-maximum gap]\label{lem:mehta-gap}
There is a polynomial-time reduction that maps an instance $(G,k)$ of
\textsc{Independent Set} on a regular graph $G=(V,E)$, with $n_G=|V|$
and $m_G=|E|$, to mutually independent, nonnegative, finite-support
random variables $(X_v)_{v\in V}$ with rational values and
probabilities of polynomial encoding length. All variables have the
same expectation $\mu$, and the following hold:
\begin{enumerate}
    \item If $G$ contains an independent set $S^\star$ of size $k$, then
    \begin{equation*}
        \mathbb E\!\left[\max_{v\in S^\star}X_v\right]
        \geq k\mu-\frac{2}{m_G}.
    \end{equation*}
    \item If $G$ contains no independent set of size $k$, then every
    $S\subseteq V$ with $|S|=k$ satisfies
    \begin{equation*}
        \mathbb E\!\left[\max_{v\in S}X_v\right]
        \leq k\mu-1.
    \end{equation*}
\end{enumerate}
\end{lemma}

We may assume without loss of generality that $m_G\geq 3$: instances
with at most two edges can be decided in polynomial time. Set
\begin{equation*}
    \Delta := 1-\frac{2}{m_G}>0.
\end{equation*}

Define
\begin{equation*} B:= k\mu-1, \qquad M:= 1+\max\{0,-B\}, \qquad U:= M+B.
\end{equation*}
Then $U\geq 1$. Shift every random variable by the same deterministic
amount:
\begin{equation*}
    Y_v:= X_v+M.
\end{equation*}
For every nonempty $S\subseteq V$,
\begin{equation*}
    \mathbb E\!\left[\max_{v\in S}Y_v\right]
    =M+\mathbb E\!\left[\max_{v\in S}X_v\right].
\end{equation*}
Consequently, Lemma~\ref{lem:mehta-gap} becomes
\begin{align}
    \text{NO case:}\quad
    &\mathbb E\!\left[\max_{v\in S}Y_v\right]\leq U
    &&\text{for every }S\subseteq V,\ |S|=k,
    \label{eq:shifted-no}
    \\
    \text{YES case:}\quad
    &\mathbb E\!\left[\max_{v\in S^\star}Y_v\right]\geq U+\Delta
    &&\text{for some independent set }S^\star,\ |S^\star|=k.
    \label{eq:shifted-yes}
\end{align}
The shift is used only to ensure that all rewards introduced below are
nonnegative.

\paragraph{MDP construction.}
For each $v\in V$, write the finite support of $Y_v$ as
\begin{equation*}
    \operatorname{supp}(Y_v)=\{y_{v,1},\ldots,y_{v,N_v}\},
    \qquad
    \mathbb{P}(Y_v=y_{v,h})=p_{v,h}.
\end{equation*}

The state space of the MDP $\mathcal M_G$ is
\begin{equation}\label{eq:state-space}
\begin{split}
    \mathcal{S}
    ={}&\{s_0,s_1,s_T\}
    \cup \{d_1,\ldots,d_{\ell-2}\}
    \cup \{s_v:v\in V\}
    \\
    &\cup \{x_{v,h}:v\in V,\ h\in[N_v]\}.
\end{split}
\end{equation}
The delay states $d_1,\ldots,d_{\ell-2}$ are absent when $\ell=2$. The
action set is
\begin{equation}\label{eq:action-space}
    \mathcal{A}
    =\{\mathsf{wait},\mathsf{go},\mathsf{advance},\mathsf{claim},\mathsf{collect}\}
      \cup\{\mathsf{pick}_1,\ldots,\mathsf{pick}_k\}.
\end{equation}

Set
\begin{equation}\label{eq:T-b}
    T:= \gamma^{\ell+1}U,
    \qquad
    b:= (1-\gamma)T.
\end{equation}
The only nonzero rewards are
\begin{equation*}
    r(s_0,\mathsf{wait})=b,
    \qquad
    r(x_{v,h},\mathsf{collect})=y_{v,h}.
\end{equation*}
All rewards are nonnegative and bounded by
$\max_{v,h}y_{v,h}$.

The transition kernel is as follows.
\begin{enumerate}
    \item At the root $s_0$,
    \begin{equation*}
        P(s_0\mid s_0,\mathsf{wait})=1.
    \end{equation*}
    If $\ell=2$, action $\mathsf{go}$ moves deterministically to $s_1$;
    if $\ell\geq 3$, it moves deterministically to $d_1$.

    \item At a delay state $d_i$, action $\mathsf{advance}$ moves
    deterministically to $d_{i+1}$ when $i<\ell-2$, and from
    $d_{\ell-2}$ to $s_1$.

    \item At the selector state $s_1$, each action $\mathsf{pick}_j$,
    $j\in[k]$, has the uniform transition law
    \begin{equation}\label{eq:uniform-pick}
        P(s_v\mid s_1,\mathsf{pick}_j)=\frac{1}{n_G},
        \qquad v\in V.
    \end{equation}

    \item At a vertex state $s_v$, action $\mathsf{claim}$ samples the
    corresponding payoff state:
    \begin{equation*}
        P(x_{v,h}\mid s_v,\mathsf{claim})=p_{v,h},
        \qquad h\in[N_v].
    \end{equation*}

    \item At a payoff state $x_{v,h}$, action $\mathsf{collect}$ moves
    deterministically to $s_T$.

    \item The terminal state $s_T$ is absorbing. Every action not
    explicitly specified above moves deterministically to $s_T$ and
    gives reward zero.
\end{enumerate}

\paragraph{Candidate tuples and the value of committing.}
We call \emph{root augmented state} any augmented state of the form
$(s_0,c)$ with $c=(\theta_0,\ldots,\theta_{\ell-1})\in\Omega^\ell$. For
$j\in[k]$, define the length-$\ell$ action sequence
\begin{equation}\label{eq:sigma-j}
    \sigma_j
    :=
    (\mathsf{go},\underbrace{\mathsf{advance},\ldots,\mathsf{advance}}_{\ell-2\text{ times}},\mathsf{pick}_j),
\end{equation}
with the obvious interpretation $(\mathsf{go},\mathsf{pick}_j)$ when
$\ell=2$. The first $\ell-1$ actions of $\sigma_j$ deterministically
drive the system from $s_0$ to $s_1$, so there is a unique
$q_j(c)\in V$ such that
\begin{equation*}
    \phi_c(s_0,\sigma_j)=s_{q_j(c)},
    \qquad
    \text{namely }
    s_{q_j(c)}=\theta_{\ell-1}(s_1,\mathsf{pick}_j).
\end{equation*}
Write
\begin{equation}\label{eq:q-S}
    q(c):=(q_1(c),\ldots,q_k(c))\in V^k,
    \qquad
    S(c):=\{q_1(c),\ldots,q_k(c)\}.
\end{equation}
When $C\sim L^{\otimes\ell}$, the coordinates of $q(C)$ are
independent and uniform on $V$, because they arise from the distinct
state--action pairs $(s_1,\mathsf{pick}_j)$ of a single
product-distributed table, via \eqref{eq:uniform-pick}.

\begin{lemma}[Commit value and root recursion]\label{lem:root-recursion-fixed}
For every root augmented state $(s_0,c)$, the value obtained by
choosing $\mathsf{go}$ at $(s_0,c)$ and then acting optimally is
\begin{equation}\label{eq:commit-value}
    C(c)
    =\gamma^{\ell+1}
      \mathbb E\!\left[\max_{v\in S(c)}Y_v\right].
\end{equation}
Moreover,
\begin{equation}\label{eq:root-recursion-fixed}
    V_\ell^\star(s_0,c)
    =
    \max\left\{
        C(c),\;
        b+\gamma\,\mathbb E\!\left[V_\ell^\star(s_0,C')\mid c,\mathsf{wait}\right]
    \right\},
\end{equation}
where $C'=(\theta_1,\ldots,\theta_{\ell-1},\Theta)$ with
$\Theta\sim L$ independent of the past. Finally, after playing
$\mathsf{wait}$, the next candidate tuple $q(C')$ is an independent
fresh draw from the uniform distribution on $V^k$.
\end{lemma}

\begin{proof}
After $\mathsf{go}$, the process traverses the deterministic delay
chain and cannot return to $s_0$. There are no rewards on this part of
the trajectory. The system reaches $s_1$ exactly $\ell-1$ steps after
leaving $s_0$. At that time, the current $\ell$-step look-ahead window
contains, for every $j\in[k]$, both
\begin{equation*}
    s_{q_j(c)}
    \quad\text{and}\quad
    x_{q_j(c),h_j}
\end{equation*}
along the continuation $(\mathsf{pick}_j,\mathsf{claim})$, since
$\ell\geq 2$. Hence the realization of $Y_{q_j(c)}$ is known before the
action $\mathsf{pick}_j$ is selected. Note also that, by prefix
consistency of the look-ahead windows, the vertex reached under
$\mathsf{pick}_j$ at that time is exactly $q_j(c)$: the transition
$(s_1,\mathsf{pick}_j)$ is governed by the last table of the root
window, which is also the table governing the system when it actually
reaches $s_1$. For distinct vertices, these payoff samples are
independent; if the same vertex appears several times in $q(c)$, prefix
consistency makes the corresponding branches share the same sample from
the pair $(s_v,\mathsf{claim})$. Therefore the selector obtains the
largest realized value among the distinct variables indexed by $S(c)$.

Moreover, the payoff realizations are not observable at the root: the
payoff states $x_{v,h}$ lie at depth $\ell+1$ from $s_0$ under
$(\sigma_j,\mathsf{claim})$, one level beyond the depth-$\ell$ tree
revealed by $c$. Hence, conditionally on the root window, the variables
$(Y_v)_{v\in S(c)}$ have their unconditional product law.

Starting from $s_1$, the reward is collected two steps later: one
transition under $\mathsf{pick}_j$, one under $\mathsf{claim}$, and
then the immediate reward of $\mathsf{collect}$ at the payoff state.
The conditional value at $s_1$ is thus
\begin{equation*}
    \gamma^2\max_{v\in S(c)}Y_v.
\end{equation*}
The selector is reached $\ell-1$ steps after leaving $s_0$, so the
value at the root is
\begin{equation*}
    \gamma^{\ell-1}\,\mathbb E\!\left[\gamma^2\max_{v\in S(c)}Y_v\right]
    =\gamma^{\ell+1}\,\mathbb E\!\left[\max_{v\in S(c)}Y_v\right],
\end{equation*}
which proves \eqref{eq:commit-value}.

At $s_0$, the only potentially optimal actions are $\mathsf{go}$ and
$\mathsf{wait}$, since every other action leads to the absorbing state
$s_T$ with zero reward. The former has value $C(c)$. The latter gives
immediate reward $b$, returns to $s_0$, and then has continuation value
$\gamma\,\mathbb E[V_\ell^\star(s_0,C')\mid c,\mathsf{wait}]$. This
proves \eqref{eq:root-recursion-fixed}.

Finally, after $\mathsf{wait}$, the candidate transitions lie one level
beyond the old look-ahead tree: $q(C')$ is determined by the entries
$(s_1,\mathsf{pick}_j)$ of the newly revealed table $\Theta$, which is
independent of all previously observed tables. The augmented transition
kernel therefore samples them independently from
\eqref{eq:uniform-pick}, independently of the previous tuple. Hence
$q(C')$ is a fresh uniform draw from $V^k$. In particular, the current
window never reveals the next candidate tuple before $\mathsf{wait}$ is
played.
\end{proof}

\begin{proof}[Proof of Theorem~\ref{thm:fixed-gamma-hardness}]
We show that the map $(G,k)\mapsto(\mathcal M_G,s_0,\gamma,\theta)$,
with $\theta$ defined in \eqref{eq:theta} below, is a valid many-one
reduction from \textsc{Independent Set} to $\ell$-\textsc{DVDP}.

\paragraph{Soundness.}
Assume that $G$ is a NO instance of \textsc{Independent Set}. Fix a
root window $c\in\Omega^\ell$. The set $S(c)$ may contain fewer than
$k$ vertices because the tuple can have repetitions. Extend it to any
set $\widetilde S\supseteq S(c)$ of cardinality $k$. Since the expected
maximum is monotone under set inclusion (the $Y_v$ are nonnegative),
\eqref{eq:shifted-no} gives
\begin{equation}\label{eq:no-small-set}
    \mathbb E\!\left[\max_{v\in S(c)}Y_v\right]
    \leq
    \mathbb E\!\left[\max_{v\in\widetilde S}Y_v\right]
    \leq U.
\end{equation}
By \eqref{eq:commit-value} and \eqref{eq:T-b},
\begin{equation}\label{eq:no-commit}
    C(c)\leq \gamma^{\ell+1}U=T.
\end{equation}

Let
\begin{equation*}
    M:=\max\{V_\ell^\star(s_0,c):c\in\Omega^\ell\}.
\end{equation*}
The set $\Omega^\ell$ is finite, so the maximum exists. From
Lemma~\ref{lem:root-recursion-fixed} and \eqref{eq:no-commit},
\begin{equation*}
    M\leq \max\{T,\;b+\gamma M\}.
\end{equation*}
If $M>T$, then necessarily
\begin{equation*}
    M\leq b+\gamma M=(1-\gamma)T+\gamma M,
\end{equation*}
which implies $(1-\gamma)M\leq(1-\gamma)T$, contradicting $M>T$.
Therefore
\begin{equation}\label{eq:no-value-bound}
    M\leq T.
\end{equation}
In particular, after averaging over the initial root window,
\begin{equation}\label{eq:no-original-value}
    v^\star_{\ell,\gamma}(s_0)\leq T.
\end{equation}

\paragraph{Completeness.}
Assume that $G$ is a YES instance, and let
\begin{equation*}
    S^\star=\{v_1,\ldots,v_k\}
\end{equation*}
be an independent set satisfying \eqref{eq:shifted-yes}. Consider the
following stationary policy on the augmented state space:
\begin{itemize}
    \item at a root augmented state $(s_0,c)$, choose $\mathsf{go}$ if
    \begin{equation*}
        q(c)=(v_1,\ldots,v_k),
    \end{equation*}
    and choose $\mathsf{wait}$ otherwise;
    \item traverse the deterministic delay chain using
    $\mathsf{advance}$;
    \item at $s_1$, choose an index whose revealed payoff is maximal;
    \item use $\mathsf{claim}$ and then $\mathsf{collect}$.
\end{itemize}
This policy is well defined: the tuple $q(c)$ is observable at the
root, and the realized payoffs are observable at $s_1$, by
Lemma~\ref{lem:root-recursion-fixed}.

At every visit to $s_0$, the target tuple appears with probability
\begin{equation}\label{eq:rho}
    \rho:= n_G^{-k}.
\end{equation}
By the last part of Lemma~\ref{lem:root-recursion-fixed}, successive
tuples are independent. Let $\tau\in\{0,1,2,\ldots\}$ be the number of
waiting actions before the target tuple first appears. Then
\begin{equation*}
    \mathbb P(\tau=t)=(1-\rho)^t\rho,
\end{equation*}
and therefore
\begin{equation}\label{eq:alpha}
    \alpha
    := \mathbb E[\gamma^\tau]
    =\sum_{t=0}^\infty \rho(1-\rho)^t\gamma^t
    =\frac{\rho}{1-\gamma(1-\rho)}.
\end{equation}

The discounted return of this policy is
\begin{equation*}
    \sum_{t=0}^{\tau-1}\gamma^t b
    +\gamma^{\tau+\ell+1}\max_{v\in S^\star}Y_v.
\end{equation*}
The payoff samples generated after commitment are independent of the
waiting time (they are drawn from table entries never used to define
the tuples) and have the laws $(Y_v)_{v\in S^\star}$, with all vertices
of the target tuple distinct. Taking expectations and using
$b/(1-\gamma)=T$ gives
\begin{align}
    v^\pi_{\ell,\gamma}(s_0)
    &=T\bigl(1-\mathbb E[\gamma^\tau]\bigr)
      +\mathbb E[\gamma^\tau]\,
       \gamma^{\ell+1}\,
       \mathbb E\!\left[\max_{v\in S^\star}Y_v\right]
       \notag\\
    &=T+\alpha\left(
       \gamma^{\ell+1}\,
       \mathbb E\!\left[\max_{v\in S^\star}Y_v\right]-T
       \right)
       \notag\\
    &\geq T+\alpha\gamma^{\ell+1}\Delta,
    \label{eq:yes-value-bound}
\end{align}
where the last inequality follows from \eqref{eq:shifted-yes} and
$T=\gamma^{\ell+1}U$.

\paragraph{Decision threshold and polynomial encoding.}
Define
\begin{equation}\label{eq:theta}
    \theta
    :=
    T+\frac12\alpha\gamma^{\ell+1}\Delta.
\end{equation}
Because $\gamma>0$, $\rho>0$, and $\Delta>0$, the added term is
strictly positive. Equations \eqref{eq:no-original-value} and
\eqref{eq:yes-value-bound} imply
\begin{align*}
    G\text{ is a NO instance}
    &\quad\Longrightarrow\quad
    v^\star_{\ell,\gamma}(s_0)\leq T<\theta,
    \\
    G\text{ is a YES instance}
    &\quad\Longrightarrow\quad
    v^\star_{\ell,\gamma}(s_0)\geq T+\alpha\gamma^{\ell+1}\Delta>\theta.
\end{align*}
Thus $(G,k)\mapsto(\mathcal M_G,s_0,\gamma,\theta)$ is a valid many-one
reduction to $\ell$-\textsc{DVDP}.

It remains to verify polynomial size. The stochastic gadget in
Lemma~\ref{lem:mehta-gap} has polynomially many support points and
polynomial-bit rational values and probabilities. The MDP has
\begin{equation*}
    O\!\left(n_G+\sum_{v\in V}N_v+\ell\right)
\end{equation*}
states and $k+5$ actions. Since $\ell$ is fixed, this is polynomial.
The shift $L$, the quantities $U,T,b,\Delta$, and the threshold
$\theta$ are obtained by a constant number of rational arithmetic
operations. Moreover, $\rho=n_G^{-k}$ has encoding length
$O(k\log n_G)$, and \eqref{eq:alpha} therefore has polynomial encoding
length. For fixed rational $\gamma$, its encoding length is constant;
if $\gamma$ is supplied in binary, all constructed numbers still have
encoding length polynomial in the combined input size. This completes
the reduction and the proof.
\end{proof}

\subsection{Proof of Theorem~\ref{thm:main-RPTAS}}
\label{proof:main-RPTAS}

\paragraph{Notation and conventions.}
We use the notation of Section~\ref{sec:perfect-lookahead-setting}: $\Omega=\mathcal S^{\mathcal S\times\mathcal A}$, $L$ is the law \eqref{eq:Q} on $\Omega$, $\mathcal T_\ell$ is the Bellman optimality operator of the augmented MDP, and $V_\ell^\star$, $Q_\ell^\star$ are its optimal value and action-value functions on $\mathcal S\times\Omega^\ell$. We write $V_{\max}:=(1-\gamma)^{-1}$, so that every value function below takes values in $[0,V_{\max}]$, and $\theta_{k:k'}:=(\theta_k,\dots,\theta_{k'})$ for windows. Throughout, $s$ denotes the root state of a query, $u$ an intermediate state, $a$ the root action, $a'$ an intermediate action, $k$ a value-iteration index, $m\in\{1,\dots,\ell\}$ a level of the backward recursion, and $i,j\in[N]:=\{1,\dots,N\}$ dictionary indices. The statement of Theorem~\ref{thm:main-RPTAS} assumes $\ell\geq2$ because exact planning is already polynomial for $\ell=1$; the proof below is valid for every $\ell\geq1$.

The proof has three steps: (i) an \emph{ideal empirical model}, defined on the full space $\mathcal S\times\Omega^\ell$, is shown to be uniformly close to the true model; (ii) its optimal value is shown to coincide, on windows made of dictionary tables, with the fixed point of the finite value iteration of Algorithm~\ref{alg:rptas-preprocessing}; (iii) the backward recursion of Algorithm~\ref{alg:rptas-query} is shown to reproduce the empirical action scores exactly on \emph{any} window. 

\paragraph{(i)}
Sample once and for all
\begin{equation}
    \Theta^1,\dots,\Theta^N \overset{\mathrm{i.i.d.}}{\sim} L,
\label{eq:dictionary-sampling}
\end{equation}
and let
\begin{equation}
    \widehat{L}_N := \frac1N\sum_{j=1}^N\delta_{\Theta^j}
\end{equation}
denote the empirical distribution of the dictionary, as in Section~\ref{subsec:rptas-offline-online-complexity}. 
The ideal empirical Bellman operator acts on $V:\mathcal S\times\Omega^\ell\to\mathbb R$ by
\begin{align}
(\widehat{\mathcal T}_\ell V)(s,\theta_{0:\ell-1}) := \max_{a\in\mathcal A} \Bigl\{ r(s,a) +\frac\gamma N\sum_{j=1}^N V\bigl(\theta_0(s,a),\theta_{1:\ell-1},\Theta^j\bigr) \Bigr\}.
\label{eq:empirical-Bellman}
\end{align}
It is a $\gamma$-contraction for the sup norm on $\mathcal S\times\Omega^\ell$; let $\widehat V^\star$ be its unique fixed point. Recall the true action-value function and define its empirical counterpart:
\begin{align}
Q_\ell^\star(s,\theta_{0:\ell-1},a)
&=
r(s,a) + \gamma\,\mathbb E_{\Theta\sim L}\!\left[ V_\ell^\star\bigl(\theta_0(s,a),\theta_{1:\ell-1}, \Theta\bigr)\right],
\label{eq:true-Q}\\
\widehat Q^\star(s,\theta_{0:\ell-1},a)
&:=
r(s,a) + \frac{\gamma}{N}\sum_{j=1}^N \widehat V^\star\bigl(\theta_0(s,a),\theta_{1:\ell-1},\Theta^j\bigr).
\label{eq:empirical-Q}
\end{align}

To compare the two models, consider the function class
\begin{equation}
    \mathcal F := \bigl\{ \Theta\mapsto V_\ell^\star(u,\theta_{1:\ell-1},\Theta): u\in\mathcal S,\ \theta_{1:\ell-1}\in\Omega^{\ell-1} \bigr\},
\end{equation}
whose cardinality satisfies $|\mathcal F| \le |\mathcal{S}||\Omega|^{\ell-1} = |\mathcal S|^{1+(\ell-1)|\mathcal{S}||\mathcal{A}|}$. Crucially, $\log|\mathcal F|=\bigl(1+(\ell-1)|\mathcal S||\mathcal A|\bigr)\log|\mathcal S|$ is polynomial in $|\mathcal{S}|$ and $|\mathcal{A}|$ for fixed $\ell$.

\begin{lemma}[Uniform concentration]
\label{lem:uniform-Hoeffding}
For every $\eta>0$,
\begin{equation}
\mathbb P\Biggl( \sup_{u,\theta_{1:\ell-1}} \biggl| \frac1N\sum_{j=1}^N V_\ell^\star(u,\theta_{1:\ell-1},\Theta^j) - \mathbb E_{\Theta\sim L} \bigl[V_\ell^\star(u,\theta_{1:\ell-1},\Theta)\bigr] \biggr|>\eta \Biggr)
\le 2|\mathcal{S}||\Omega|^{\ell-1}
\exp\left(-\frac{2N\eta^2}{V_{\max}^2}\right).
\label{eq:uniform-Hoeffding-bound}
\end{equation}
\end{lemma}

\begin{proof}
$V_\ell^\star$ is deterministic and does not depend on the sampled dictionary, so each function in $\mathcal F$ is fixed and takes values in $[0,V_{\max}]$. Apply Hoeffding's inequality to each of them and take a union bound over $\mathcal F$.
\end{proof}

Denote by $\mathcal E_\eta$ the complement of the event in \eqref{eq:uniform-Hoeffding-bound}, i.e.\ the event on which the supremum is at most $\eta$.

\begin{lemma}[Fixed-point stability]
\label{lemma:fixed-point-perturbation}
On $\mathcal{E}_\eta$,
\begin{equation}
\|\widehat V^\star-V_\ell^\star\|_\infty
\le
\frac{\gamma\eta}{1-\gamma}.
\end{equation}
\end{lemma}

\begin{proof}
Fix $(s,\theta_{0:\ell-1})$. For every $a\in\mathcal A$, the bracketed terms in $\widehat{\mathcal T}_\ell V_\ell^\star$ and $\mathcal T_\ell V_\ell^\star$ share the reward $r(s,a)$ and differ by $\gamma$ times an empirical-versus-true average of $V_\ell^\star(\theta_0(s,a),\theta_{1:\ell-1},\cdot)\in\mathcal F$, which is at most $\gamma\eta$ in absolute value on $\mathcal E_\eta$. Since the maximum over $a$ is $1$-Lipschitz,
$\|\widehat{\mathcal T}_\ell V_\ell^\star -\mathcal T_\ell V_\ell^\star\|_\infty \le \gamma\eta$.
As $V_\ell^\star=\mathcal T_\ell V_\ell^\star$, $\widehat V^\star=\widehat{\mathcal T}_\ell\widehat V^\star$, and $\widehat{\mathcal T}_\ell$ is a $\gamma$-contraction,
\begin{align*}
\|\widehat V^\star-V_\ell^\star\|_\infty
&\le
\|\widehat{\mathcal T}_\ell\widehat V^\star
-\widehat{\mathcal T}_\ell V_\ell^\star\|_\infty
+
\|\widehat{\mathcal T}_\ell V_\ell^\star
-\mathcal T_\ell V_\ell^\star\|_\infty
\le
\gamma\|\widehat V^\star-V_\ell^\star\|_\infty
+
\gamma\eta,
\end{align*}
and rearranging proves the claim.
\end{proof}

\begin{lemma}[Action-value perturbation]
\label{lem:Q-perturbation}
On $\mathcal E_\eta$,
\begin{equation}
\sup_{s,\theta_{0:\ell-1},a}
\bigl|\widehat Q^\star(s,\theta_{0:\ell-1},a)-Q_\ell^\star(s,\theta_{0:\ell-1},a)\bigr| \le \frac{\gamma\eta}{1-\gamma}.
\label{eq:Q-perturbation}
\end{equation}
\end{lemma}

\begin{proof}
Adding and subtracting the empirical average of $V_\ell^\star$,
\begin{align*}
&\bigl|\widehat Q^\star(s,\theta_{0:\ell-1},a)-Q_\ell^\star(s,\theta_{0:\ell-1},a)\bigr|
\\&\le \frac{\gamma}{N}\sum_{j=1}^N \bigl| \widehat V^\star\bigl(\theta_0(s,a),\theta_{1:\ell-1},\Theta^j\bigr) - V_\ell^\star\bigl(\theta_0(s,a),\theta_{1:\ell-1},\Theta^j\bigr) \bigr|  + \gamma\eta\\
&\le \gamma\|\widehat V^\star-V_\ell^\star\|_\infty + \gamma\eta
\le \frac{\gamma\eta}{1-\gamma},
\end{align*}
where the last step uses Lemma~\ref{lemma:fixed-point-perturbation} and $\gamma\cdot\frac{\gamma\eta}{1-\gamma}+\gamma\eta = \frac{\gamma\eta}{1-\gamma}$.
\end{proof}

\paragraph{(ii)}
The function $\widehat V^\star$ is close to $V_\ell^\star$ but is still defined on the exponentially large space $\mathcal S\times\Omega^\ell$. We now show that windows composed only of dictionary tables form a finite state space preserved by $\widehat{\mathcal T}_\ell$, on which $\widehat V^\star$ can be computed by ordinary value iteration. The results of this step and of Step~3 hold for \emph{any} fixed dictionary $(\Theta^1,\dots,\Theta^N)\in\Omega^N$, whatever its law.

Recall the proxy state space $\mathcal D_N:=\mathcal S\times[N]^\ell$ of Section~\ref{subsec:rptas-offline-online-complexity} and define the embedding
\begin{equation}
\iota:\mathcal D_N\to\mathcal S\times\Omega^\ell,
\qquad
\iota(s,i_0,\dots,i_{\ell-1}):=(s,\Theta^{i_0},\dots,\Theta^{i_{\ell-1}}),
\label{eq:dictionary-embedding}
\end{equation}
which reads an index tuple as the window of dictionary tables it encodes, the index $i_m$ specifying the table at position $m$ of the window. For $V:\mathcal S\times\Omega^\ell\to\mathbb R$, let $R_NV:=V\circ\iota:\mathcal D_N\to\mathbb R$ denote its restriction to dictionary windows. The proxy Bellman operator acts on arrays $v:\mathcal D_N\to\mathbb R$ by
\begin{align}
(\widehat{\mathcal T}_Nv)(s,i_0,\dots,i_{\ell-1}) := \max_{a\in\mathcal A} \Bigl\{ r(s,a) +\frac\gamma N\sum_{j=1}^N v\bigl(\Theta^{i_0}(s,a),i_1,\dots,i_{\ell-1},j\bigr)
\Bigr\};
\label{eq:core-Bellman}
\end{align}
this is exactly the update \eqref{eq:vi-finite-statespace} of Algorithm~\ref{alg:rptas-preprocessing}. It is a $\gamma$-contraction on $\mathbb R^{\mathcal D_N}$; let $V^\star_N$ denote its unique fixed point.

\begin{lemma}[Exact restriction to the dictionary]
\label{lem:core-restriction}
$V^\star_N=R_N\widehat V^\star$, that is, for every $(s,i_0,\ldots,i_{\ell-1})\in\mathcal{D}_N$,
\begin{equation}
V^\star_N(s,i_0,\ldots,i_{\ell-1}) = \widehat V^\star(s,\Theta^{i_0},\ldots,\Theta^{i_{\ell-1}}).
\label{eq:core-restriction}
\end{equation}
\end{lemma}

\begin{proof}
We first show that $R_N\widehat{\mathcal T}_\ell V=\widehat{\mathcal T}_N R_NV$ for every $V:\mathcal S\times\Omega^\ell\to\mathbb R$. Fix $(s,i_0,\ldots,i_{\ell-1})\in\mathcal D_N$. By \eqref{eq:empirical-Bellman} and \eqref{eq:dictionary-embedding},
\begin{align*}
(R_N\widehat{\mathcal T}_\ell V) (s,i_0,\ldots,i_{\ell-1})
&= (\widehat{\mathcal T}_\ell V) (s,\Theta^{i_0},\ldots,\Theta^{i_{\ell-1}})
\\
&= \max_{a\in\mathcal A} \Bigl\{ r(s,a) + \frac{\gamma}{N} \sum_{j=1}^N V\bigl( \Theta^{i_0}(s,a), \Theta^{i_1},\ldots,\Theta^{i_{\ell-1}},\Theta^j \bigr)
\Bigr\}
\\
&=
\max_{a\in\mathcal A} \Bigl\{ r(s,a) + \frac{\gamma}{N} \sum_{j=1}^N (R_NV)\bigl( \Theta^{i_0}(s,a),i_1,\ldots,i_{\ell-1},j \bigr) \Bigr\}
\\&= (\widehat{\mathcal T}_N R_NV) (s,i_0,\ldots,i_{\ell-1}),
\end{align*}
where the third equality holds because the window $(\Theta^{i_1},\ldots,\Theta^{i_{\ell-1}},\Theta^j)$ consists of dictionary tables. Applying this identity to $V=\widehat V^\star$ and using $\widehat{\mathcal T}_\ell\widehat V^\star=\widehat V^\star$ gives
$\widehat{\mathcal T}_N(R_N\widehat V^\star) = R_N(\widehat{\mathcal T}_\ell\widehat V^\star) = R_N\widehat V^\star$.
Thus $R_N\widehat V^\star$ is a fixed point of the $\gamma$-contraction $\widehat{\mathcal T}_N$, whose fixed point is unique; hence $R_N\widehat V^\star=V^\star_N$.
\end{proof}

Algorithm~\ref{alg:rptas-preprocessing} computes $V^\star_N$ approximately by value iteration,
\begin{equation}
V^0:=0, \qquad V^{k+1}:=\widehat{\mathcal T}_N V^k,\qquad k=0,\dots,K-1,
\label{eq:value-iteration}
\end{equation}
and by contraction $\|V^K-V^\star_N\|_\infty\le\gamma^KV_{\max}$.

\begin{remark}[Cost and memory of one iteration]
\label{rem:sweep-cost}
For an array $v:\mathcal D_N\to\mathbb R$ and every $(u,i_1,\dots,i_{\ell-1})\in\mathcal S\times[N]^{\ell-1}$, first compute
\begin{equation}
\mathrm{Mean}_v(u,i_1,\dots,i_{\ell-1}) :=\frac1N\sum_{j=1}^N v(u,i_1,\dots,i_{\ell-1},j),
\label{eq:Ax}
\end{equation}
which costs $O(|\mathcal{S}|N^\ell)$ in total; then each of the $|\mathcal S|N^\ell$ maxima in \eqref{eq:core-Bellman} costs $O(|\mathcal A|)$ look-ups into $\mathrm{Mean}_v$. One iteration of \eqref{eq:core-Bellman} therefore costs
\begin{equation}
O\bigl(|\mathcal{S}|(|\mathcal{A}|+1)N^\ell\bigr),
\label{eq:sweep-cost}
\end{equation}
and the arrays $v$ and $\mathrm{Mean}_v$ occupy $O(|\mathcal S|N^\ell)$ memory.
\end{remark}

\paragraph{(iii)}
$V^\star_N$ is only defined on dictionary windows, whereas the window observed at deployment is an arbitrary $c=(\theta_0,\ldots,\theta_{\ell-1})\in\Omega^\ell$. The key observation is that after $\ell$ transitions every table of $c$ has been shifted out and replaced by a fresh table, which the empirical model draws from the dictionary. Algorithm~\ref{alg:rptas-query} therefore solves an $\ell$-step planning problem with terminal values on $\mathcal D_N$, reintroducing the observed tables $\theta_{\ell-1},\ldots,\theta_0$ by backward recursion. We define this recursion for an arbitrary terminal array so that we can analyse both the exact array $V^\star_N$ and the computed array $V^K$.

Fix a window $c=(\theta_0,\ldots,\theta_{\ell-1})\in\Omega^\ell$ and an array $v:\mathcal D_N\to\mathbb R$. Define $U^v_{m,c}:\mathcal S\times[N]^m\to\mathbb R$ for $m=\ell,\ell-1,\dots,1$ by the terminal condition
\begin{equation}
U^v_{\ell,c}(u,i_0,\ldots,i_{\ell-1}) := v(u,i_0,\ldots,i_{\ell-1}),
\label{eq:extension-terminal}
\end{equation}
and, for $m=\ell-1,\ldots,1$,
\begin{align}
U^v_{m,c}(u,i_0,\ldots,i_{m-1}) := \max_{a'\in\mathcal A} \Bigl\{ r(u,a') +\frac{\gamma}{N}\sum_{i_m=1}^N U^v_{m+1,c}\bigl(\theta_m(u,a'),i_0,\ldots,i_m\bigr) \Bigr\}.
\label{eq:extension-recursion}
\end{align}
Here $U^v_{m,c}(u,i_0,\dots,i_{m-1})$ is the continuation value at step $m$ of the query, when the remaining observed tables are $\theta_{m:\ell-1}$ and the $m$ tables revealed since the query are represented by the dictionary indices $i_0,\dots,i_{m-1}$, in order of arrival. Finally, the action scores at the root are
\begin{equation}
\widetilde Q_v(s,c,a) := r(s,a) + \frac{\gamma}{N}\sum_{i_0=1}^N U^v_{1,c}\bigl(\theta_0(s,a),i_0\bigr),
\qquad a\in\mathcal A.
\label{eq:approx-Q-extension}
\end{equation}
With $v=V^K$, \eqref{eq:extension-terminal}--\eqref{eq:approx-Q-extension} are exactly the computations of Algorithm~\ref{alg:rptas-query}, and $\pi_D(s,c)\in\arg\max_a\widetilde Q_{V^K}(s,c,a)$.

\begin{lemma}[Exact extension identity]
\label{lem:extension-identity}
For every $(s,c,a)\in\mathcal S\times\Omega^\ell\times\mathcal A$,
\begin{equation}
\widetilde Q_{V^\star_N}(s,c,a) = \widehat Q^\star(s,c,a).
\label{eq:extension-main-identity}
\end{equation}
\end{lemma}

\begin{proof}
Fix $c=(\theta_0,\ldots,\theta_{\ell-1})\in\Omega^\ell$ and write $U_m:=U^{V^\star_N}_{m,c}$. We first prove, by backward induction on $m\in\{1,\ldots,\ell\}$, that for all $u\in\mathcal S$ and $i_0,\dots,i_{m-1}\in[N]$,
\begin{equation}
U_m(u,i_0,\ldots,i_{m-1}) = \widehat V^\star \bigl( u,\theta_{m:\ell-1}, \Theta^{i_0},\ldots,\Theta^{i_{m-1}} \bigr).
\label{eq:extension-intermediate-identity}
\end{equation}
For $m=\ell$, the terminal condition \eqref{eq:extension-terminal} and Lemma~\ref{lem:core-restriction} give
$U_\ell(u,i_0,\ldots,i_{\ell-1})=V^\star_N(u,i_0,\ldots,i_{\ell-1})=\widehat V^\star(u,\Theta^{i_0},\ldots,\Theta^{i_{\ell-1}})$, which is \eqref{eq:extension-intermediate-identity} at level $\ell$ (the list $\theta_{\ell:\ell-1}$ being empty). Suppose \eqref{eq:extension-intermediate-identity} holds at level $m+1$ for some $m\in\{1,\ldots,\ell-1\}$. By \eqref{eq:extension-recursion} and the induction hypothesis,
\begin{align*}
U_m(u,i_0,\ldots,i_{m-1})
&= \max_{a'\in\mathcal A} \Bigl\{ r(u,a') + \frac{\gamma}{N} \sum_{i_m=1}^N \widehat V^\star \bigl( \theta_m(u,a'), \theta_{m+1:\ell-1}, \Theta^{i_0},\ldots,\Theta^{i_m} \bigr) \Bigr\}\\
&= (\widehat{\mathcal T}_\ell\widehat V^\star) \bigl( u,\theta_{m:\ell-1}, \Theta^{i_0},\ldots,\Theta^{i_{m-1}} \bigr)
= \widehat V^\star \bigl( u,\theta_{m:\ell-1}, \Theta^{i_0},\ldots,\Theta^{i_{m-1}} \bigr),
\end{align*}
where the second equality is the definition \eqref{eq:empirical-Bellman} applied to the window $(\theta_m,\theta_{m+1:\ell-1},\Theta^{i_0},\dots,\Theta^{i_{m-1}})\in\Omega^\ell$, whose first table is $\theta_m$ and whose appended table is $\Theta^{i_m}$, and the last equality uses $\widehat{\mathcal T}_\ell\widehat V^\star=\widehat V^\star$. This proves \eqref{eq:extension-intermediate-identity} for every $m$. Finally, applying \eqref{eq:extension-intermediate-identity} at level $m=1$ in \eqref{eq:approx-Q-extension},
\begin{align*}
\widetilde Q_{V^\star_N}(s,c,a)
= r(s,a) + \frac{\gamma}{N} \sum_{i_0=1}^N \widehat V^\star \bigl( \theta_0(s,a), \theta_{1:\ell-1}, \Theta^{i_0} \bigr)
= \widehat Q^\star(s,c,a),
\end{align*}
by the definition \eqref{eq:empirical-Q}.
\end{proof}

\begin{lemma}[Stability of the extension]
\label{lem:extension-Lipschitz}
For all arrays $v_1,v_2:\mathcal{D}_N\to\mathbb R$ and every level $m\in\{1,\ldots,\ell\}$,
\begin{equation}
\sup_{c,u,i_0,\ldots,i_{m-1}}
\bigl|
U^{v_1}_{m,c}(u,i_0,\ldots,i_{m-1})
-
U^{v_2}_{m,c}(u,i_0,\ldots,i_{m-1})
\bigr|
\le
\gamma^{\ell-m}\|v_1-v_2\|_\infty.
\label{eq:extension-Lipschitz-all-k}
\end{equation}
Consequently,
\begin{equation}
\sup_{s,c,a}\bigl|\widetilde Q_{v_1}(s,c,a)-\widetilde Q_{v_2}(s,c,a)\bigr|
\le
\gamma^\ell\|v_1-v_2\|_\infty.
\label{eq:extension-Q-Lipschitz}
\end{equation}
\end{lemma}

\begin{proof}
Fix $c$; the recursion \eqref{eq:extension-recursion} never modifies it, and all bounds below are uniform in $c$. We prove \eqref{eq:extension-Lipschitz-all-k} by backward induction on $m$. For $m=\ell$, \eqref{eq:extension-terminal} gives $U^{v_1}_{\ell,c}-U^{v_2}_{\ell,c}=v_1-v_2$ on $\mathcal D_N$, so \eqref{eq:extension-Lipschitz-all-k} holds with equality. Assume it at level $m+1$ and fix $u,i_0,\ldots,i_{m-1}$. In \eqref{eq:extension-recursion}, the two braces associated with $v_1$ and $v_2$ share the term $r(u,a')$, so for every $a'$ their difference is
\begin{equation*}
\frac{\gamma}{N}\sum_{i_m=1}^{N}
\bigl[U^{v_1}_{m+1,c}-U^{v_2}_{m+1,c}\bigr]
\bigl(\theta_m(u,a'),i_0,\ldots,i_m\bigr),
\end{equation*}
whose absolute value is at most $\gamma\cdot\gamma^{\ell-m-1}\|v_1-v_2\|_\infty$ by the induction hypothesis. Since maximizing over $a'$ is $1$-Lipschitz, the same bound holds for $|U^{v_1}_{m,c}-U^{v_2}_{m,c}|(u,i_0,\ldots,i_{m-1})$. For \eqref{eq:extension-Q-Lipschitz}, the term $r(s,a)$ cancels in \eqref{eq:approx-Q-extension}, so $|\widetilde Q_{v_1}-\widetilde Q_{v_2}|(s,c,a)$ is $\gamma$ times an average of differences at level $1$, hence at most $\gamma\cdot\gamma^{\ell-1}\|v_1-v_2\|_\infty$.
\end{proof}

\begin{remark}[Deployment cost]
\label{rem:deployment-cost}
Proceeding level by level as in \eqref{eq:Ax}, computing $U^v_{m,c}$ from $U^v_{m+1,c}$ costs $O(|\mathcal S|N^{m+1})$ for the averages over $i_m$ and $O(|\mathcal S||\mathcal A|N^{m})$ for the maxima, and \eqref{eq:approx-Q-extension} costs $O(|\mathcal S|N+|\mathcal A|)$. The cost of computing all action scores for one queried pair $(s,c)$ is therefore
\begin{equation}
O\Bigl( \sum_{m=0}^{\ell-1} \bigl(|\mathcal{S}|N^{m+1}+|\mathcal{S}||\mathcal{A}|N^m\bigr) \Bigr)
= O\bigl(\ell|\mathcal S|N^\ell+\ell|\mathcal S||\mathcal A|N^{\ell-1}\bigr)
\subseteq O\bigl(\ell |\mathcal{S}|(|\mathcal{A}|+1)N^\ell\bigr).
\label{eq:extension-cost}
\end{equation}
The terminal array $v$, of size $|\mathcal S|N^\ell$, is already stored. During a query, the largest temporary array is $U^v_{\ell-1,c}$, with $|\mathcal S|N^{\ell-1}$ entries, and $U^v_{m+1,c}$ can be discarded once $U^v_{m,c}$ has been computed. The query thus needs only $O(|\mathcal S|N^{\ell-1})$ additional working memory.
\end{remark}

\paragraph{Conclusion.}
Run Algorithm~\ref{alg:rptas-preprocessing} with
\begin{equation}
N := \left\lceil \frac{8}{\varepsilon^2(1-\gamma)^4}  \Bigl( \log\frac{2}{\delta} + \bigl(1+(\ell-1)|\mathcal{S}||\mathcal{A}|\bigr)\log |\mathcal{S}| \Bigr) \right\rceil,
\qquad
K := \left\lceil \frac{1}{1-\gamma} \log\frac{4}{\varepsilon(1-\gamma)} \right\rceil.
\label{eq:rptas-parameters}
\end{equation}
It returns $D=(\Theta^1,\ldots,\Theta^N,V^K)$, and Algorithm~\ref{alg:rptas-query} defines the policy $\pi_D$, greedy with respect to $\widetilde Q_{V^K}$.

\emph{Model error.} Set
\begin{equation}
    \eta := \frac{\varepsilon(1-\gamma)^2V_{\max}}{4}=\frac{\varepsilon(1-\gamma)}{4}.
\label{eq:eta-choice}
\end{equation}
Since $V_{\max}^2/(2\eta^2)=8/(\varepsilon^2(1-\gamma)^4)$, the choice of $N$ in \eqref{eq:rptas-parameters} makes the right-hand side of \eqref{eq:uniform-Hoeffding-bound} at most $\delta$, so $\mathbb P(\mathcal E_\eta)\ge1-\delta$. On $\mathcal E_\eta$, Lemma~\ref{lem:Q-perturbation} gives
\begin{equation}
\sup_{s,c,a}
\bigl| \widehat Q^\star(s,c,a)-Q_\ell^\star(s,c,a) \bigr|
\leq \frac{\gamma\eta}{1-\gamma}
\leq \frac{\varepsilon(1-\gamma)V_{\max}}{4}.
\label{eq:final-model-error}
\end{equation}

\emph{Computation error.} By contraction of $\widehat{\mathcal T}_N$ and the choice of $K$,
\begin{equation*}
\|V^K-V^\star_N\|_\infty \leq \gamma^K V_{\max}\leq e^{-(1-\gamma)K}V_{\max}\leq \frac{\varepsilon(1-\gamma)V_{\max}}{4}.
\end{equation*}
Lemma~\ref{lem:extension-identity} identifies $\widetilde Q_{V^\star_N}$ with $\widehat Q^\star$, and Lemma~\ref{lem:extension-Lipschitz} then gives
\begin{align}
    \sup_{s,c,a} \bigl| \widetilde Q_{V^K}(s,c,a)-\widehat Q^\star(s,c,a) \bigr|
&= \sup_{s,c,a} \bigl| \widetilde Q_{V^K}(s,c,a)-\widetilde Q_{V^\star_N}(s,c,a) \bigr|
\\&\leq \gamma^\ell\|V^K-V^\star_N\|_\infty
\\& \leq \frac{\varepsilon(1-\gamma)V_{\max}}{4}.
\label{eq:final-computation-error}
\end{align}
This bound is deterministic. Combining \eqref{eq:final-model-error} and \eqref{eq:final-computation-error}, on $\mathcal E_\eta$,
\begin{equation}
\sup_{s,c,a} \bigl| \widetilde Q_{V^K}(s,c,a)-Q_\ell^\star(s,c,a) \bigr| \leq \frac{\varepsilon(1-\gamma)V_{\max}}{2}.
\label{eq:final-score-error}
\end{equation}

\emph{From scores to values.} Fix $(s,c)$, let $a^\star$ maximize $Q_\ell^\star(s,c,\cdot)$, and recall that $\pi_D(s,c)$ maximizes $\widetilde Q_{V^K}(s,c,\cdot)$. By \eqref{eq:final-score-error},
\begin{align*}
V_\ell^\star(s,c) = Q_\ell^\star(s,c,a^\star)
&\leq \widetilde Q_{V^K}(s,c,a^\star) + \frac{\varepsilon(1-\gamma)V_{\max}}{2}
\\& \leq \widetilde Q_{V^K}(s,c,\pi_D(s,c)) + \frac{\varepsilon(1-\gamma)V_{\max}}{2}\\
&\leq Q_\ell^\star(s,c,\pi_D(s,c)) + \varepsilon(1-\gamma)V_{\max}.
\end{align*}
Let $\mathcal T_\ell^{\pi_D}$ denote the Bellman evaluation operator of $\pi_D$ in the true augmented MDP. Since $Q^\star_\ell(s,c,\pi_D(s,c))=(\mathcal T_\ell^{\pi_D}V^\star_\ell)(s,c)$, the preceding inequality reads
$V_\ell^\star\leq\mathcal T_\ell^{\pi_D}V_\ell^\star+\varepsilon(1-\gamma)V_{\max}\mathbf 1$.
As $V_\ell^{\pi_D} =\mathcal T_\ell^{\pi_D}V_\ell^{\pi_D}$ and $\mathcal T_\ell^{\pi_D}$ is a monotone $\gamma$-contraction,
\begin{equation*}
\|V_\ell^\star-V_\ell^{\pi_D}\|_\infty
\leq
\gamma\|V_\ell^\star-V_\ell^{\pi_D}\|_\infty + \varepsilon(1-\gamma)V_{\max},
\qquad\text{hence}\qquad
\|V_\ell^\star-V_\ell^{\pi_D}\|_\infty\leq\varepsilon V_{\max}.
\end{equation*}
This proves \eqref{eq:main-policy-guarantee} simultaneously for every $(s,c)\in\mathcal S\times\Omega^\ell$, on the event $\mathcal E_\eta$ of probability at least $1-\delta$.

\emph{Computational guarantees.} Sampling the dictionary costs $O(N|\mathcal{S}||\mathcal{A}|)$ and, by Remark~\ref{rem:sweep-cost}, the $K$ value-iteration steps cost $O(K|\mathcal{S}|(|\mathcal{A}|+1)N^\ell)$, so preprocessing costs
$O\bigl(N|\mathcal{S}||\mathcal{A}|+K|\mathcal{S}|(|\mathcal{A}|+1)N^\ell\bigr)$.
The representation $D=(\Theta^1,\ldots,\Theta^N,V^K)$ occupies $O(N|\mathcal{S}||\mathcal{A}|+|\mathcal{S}|N^\ell)$ memory locations. By Remark~\ref{rem:deployment-cost}, selecting an action at a queried pair $(s,c)$ costs $O(\ell|\mathcal{S}|(|\mathcal{A}|+1)N^\ell)$ time and $O(|\mathcal S|N^{\ell-1})$ working memory, the observed window itself occupying $O(\ell|\mathcal S||\mathcal A|)$; the total memory during online action selection is therefore
$O\bigl(N|\mathcal{S}||\mathcal{A}|+|\mathcal{S}|N^\ell+\ell |\mathcal{S}||\mathcal{A}|\bigr)$.
Finally, by \eqref{eq:rptas-parameters},
\begin{equation*}
N
= O\left( \frac{ \log(1/\delta) + \bigl(1+(\ell-1)|\mathcal{S}||\mathcal{A}|\bigr)\log |\mathcal{S}| }{ \varepsilon^2(1-\gamma)^4 } \right), \qquad K = O\left( \frac{ \log\bigl(1/[\varepsilon(1-\gamma)]\bigr) }{ 1-\gamma } \right),
\end{equation*}
so for every fixed $\ell$, preprocessing time, representation size, and per-decision computation are polynomial in $|\mathcal{S}|$, $|\mathcal{A}|$, $\varepsilon^{-1}$, $\log(1/\delta)$, and $(1-\gamma)^{-1}$. This completes the proof. \hfill$\square$

\subsection{Proof of theorem \ref{thm:learning}}
\label{proof:regret}

Throughout this section, $M$ denotes the number of observed tables used at a planner update, $\Theta_0,\ldots,\Theta_{M-1}$ these tables, and $\widehat P_M$, $\widehat L_M=\bigotimes_{(s,a)}\widehat P_M(\cdot\mid s,a)$ the empirical model of Section~\ref{sec:regret-minimization}, equation~\eqref{eq:online-product-model}.
Note that $\widehat L_M$ is a \emph{product} law, not the empirical distribution of the dictionary used in Appendix~\ref{proof:main-RPTAS}. We write $V_{\max}:=(1-\gamma)^{-1}$ and use the shorthand $\theta_{k:k'}=(\theta_k,\ldots,\theta_{k'})$.

Let $\widehat V^\star_M$ and $\widehat Q^\star_M$ denote the optimal value and action-value functions of the augmented MDP on $\mathcal S\times\Omega^\ell$ when each incoming table has law $\widehat L_M$. At an update, RPTAS-\textsc{offline} draws
\begin{equation*}
Z^1,\ldots,Z^{N_M}\overset{\text{i.i.d.}}{\sim}\widehat L_M,
\qquad
N_M:=\bigl\lceil MV_{\max}^3\bigr\rceil,
\end{equation*}
runs $K_M:=\lceil V_{\max}\log(V_{\max}M^2)\rceil$ value-iteration steps from zero on $\mathcal S\times[N_M]^\ell$, and returns $D_M=(Z^1,\ldots,Z^{N_M},v^{K_M})$. Given a queried pair $(s,c)$, RPTAS-\textsc{online} returns the action scores $\widetilde Q_M(s,c,\cdot)$ obtained from \eqref{eq:extension-recursion}--\eqref{eq:approx-Q-extension} with dictionary $Z^1,\ldots,Z^{N_M}$ and terminal array $v^{K_M}$. The deployed policy $\pi_M$ is greedy with respect to $\widetilde Q_M$ with a fixed tie-breaking rule; it is a deterministic stationary policy on the augmented state space. At time $0$, the algorithm executes a fixed action $A_0$; we let $\pi_0$ denote the constant policy that always plays $A_0$.

For $t\geq1$, let $M_t:=2^{\lfloor\log_2 t\rfloor}$ be the number of observed tables at the most recent update before time $t$, and set $W_t:=V_\ell^{\pi_{M_t}}$ for $t\geq1$, $W_0:=V_\ell^{\pi_0}$. Thus $W_t$ is the true infinite-horizon value of the policy deployed at time $t$, and is constant as a function between two updates. The epochs are $\{0\}$ and the nonempty sets $\{M,\ldots,\min(2M-1,T-1)\}$ for dyadic $M<T$; their number is at most
\begin{equation*}
J:=2+\lceil\log_2T\rceil.
\end{equation*}
We abbreviate the two logarithmic factors that appear below as
\begin{equation}
\Lambda:=\log\frac{16J|\mathcal S||\Omega|^{\ell-1}}{\delta},
\qquad
\Lambda_0:=|\mathcal S||\mathcal A|(|\mathcal S|-1)\log 2+\log\frac{16J}{\delta},
\label{eq:lazy-log-factors}
\end{equation}
and define the burn-in size
\begin{equation}
M_0:=\bigl\lceil 16V_{\max}^2\Lambda_0\bigr\rceil.
\label{eq:lazy-M0}
\end{equation}
Since $|\Omega|=|\mathcal S|^{|\mathcal S||\mathcal A|}$, we have $\Lambda=\log(16J/\delta)+\bigl(1+(\ell-1)|\mathcal S||\mathcal A|\bigr)\log|\mathcal S|$.

Finally, for a deterministic stationary policy $\pi$ and an augmented state $x$, we write $\mathbb E^\pi_x$ for the expectation along the auxiliary augmented chain $(\bar X_h)_{h\geq0}$ with $\bar X_0=x$, actions chosen by $\pi$, and each incoming table drawn independently from the true law $L$. We write $\bar X_h=(\bar S_h,\bar\Theta_{h:h+\ell-1})$, so that the tables $\bar\Theta_{0:\ell-1}$ are those of $x$ and $\bar\Theta_{h+\ell}\sim L$ is the table appended at step $h$. This chain is a device for the analysis and is distinct from the trajectory $(X_t)_{t\geq0}$ generated by the algorithm.

Adding and subtracting $(1-\gamma)W_t(X_t)$ in the definition of regret gives
\begin{align}
\operatorname{Reg}(T)
&=
\sum_{t=0}^{T-1}
\bigl[(1-\gamma)V_\ell^\star(X_t)-r_t\bigr]
\nonumber\\
&=
\underbrace{
(1-\gamma)\sum_{t=0}^{T-1}
\bigl[V_\ell^\star(X_t)-W_t(X_t)\bigr]
}_{\text{policy loss}}
+
\underbrace{
\sum_{t=0}^{T-1}
\bigl[(1-\gamma)W_t(X_t)-r_t\bigr]
}_{\text{value-to-reward error}}.
\label{eq:lazy-final-decomposition}
\end{align}
We bound the two terms in turn. The policy loss is controlled by Lemmas~\ref{lem:lazy-product-bernstein}--\ref{lem:lazy-dictionary}, the value-to-reward error by Lemma~\ref{lem:lazy-epoch-martingale}.

Fix an update based on $M$ observed tables. Adding and subtracting $\widehat V^\star_M$ gives
\begin{equation}
V_\ell^\star-V_\ell^{\pi_M} = (V_\ell^\star-\widehat V^\star_M) + (\widehat V^\star_M-V_\ell^{\pi_M}).
\label{eq:lazy-policy-value-split}
\end{equation}
For $x=(s,\theta_{0:\ell-1})$, the Bellman equations give
\begin{align*}
\widehat Q^\star_M(x,\pi_M(x)) &= r(s,\pi_M(x)) +\gamma\,\mathbb E_{\Theta\sim\widehat L_M} \bigl[ \widehat V^\star_M \bigl(\theta_0(s,\pi_M(x)),\theta_{1:\ell-1},\Theta\bigr) \bigr],\\
V_\ell^{\pi_M}(x)
&= r(s,\pi_M(x)) +\gamma\,\mathbb E_{\Theta\sim L} \bigl[ V_\ell^{\pi_M} \bigl(\theta_0(s,\pi_M(x)),\theta_{1:\ell-1},\Theta\bigr) \bigr].
\end{align*}
Subtracting the second equality from the first, adding $\widehat V^\star_M(x)-\widehat Q^\star_M(x,\pi_M(x))$ to both sides, adding and subtracting $\gamma\,\mathbb E_{\Theta\sim L}[\widehat V^\star_M(\cdot)]$, and writing $\widehat V^\star_M=V_\ell^\star+(\widehat V^\star_M-V_\ell^\star)$ in the difference of expectations yields
\begin{align}
\widehat V^\star_M(x)-V_\ell^{\pi_M}(x)
&=
\widehat V^\star_M(x)-\widehat Q^\star_M(x,\pi_M(x))
\nonumber\\
&\quad+
\gamma\,\mathbb E_{\Theta\sim L}
\bigl[
(\widehat V^\star_M-V_\ell^{\pi_M})
\bigl(\theta_0(s,\pi_M(x)),\theta_{1:\ell-1},\Theta\bigr)
\bigr]
\nonumber\\
&\quad+
\gamma
\bigl(
\mathbb E_{\Theta\sim\widehat L_M}
-\mathbb E_{\Theta\sim L}
\bigr)
\bigl[
V_\ell^\star
\bigl(\theta_0(s,\pi_M(x)),\theta_{1:\ell-1},\Theta\bigr)
\bigr]
\nonumber\\
&\quad+
\gamma
\bigl(
\mathbb E_{\Theta\sim\widehat L_M}
-\mathbb E_{\Theta\sim L}
\bigr)
\bigl[
(\widehat V^\star_M-V_\ell^\star)
\bigl(\theta_0(s,\pi_M(x)),\theta_{1:\ell-1},\Theta\bigr)
\bigr].
\label{eq:lazy-policy-one-step}
\end{align}
Fix the observed tables and the planner's randomization, so that $\widehat L_M$, $\widehat V^\star_M$ and $\pi_M$ are fixed. Repeatedly substituting \eqref{eq:lazy-policy-one-step} into its own continuation term along the auxiliary chain $(\bar X_h)$ under $\mathbb E^{\pi_M}_x$, and applying the tower property, gives for every $H\geq1$,
\begin{align}
\widehat V^\star_M(x)-V_\ell^{\pi_M}(x)
&= \sum_{h=0}^{H-1}\gamma^h\, \mathbb E^{\pi_M}_x\bigl[ \widehat V^\star_M(\bar X_h) -\widehat Q^\star_M(\bar X_h,\pi_M(\bar X_h)) \bigr] \nonumber\\
&\quad+ \sum_{h=0}^{H-1}\gamma^{h+1}\, \mathbb E^{\pi_M}_x\Bigl[ \bigl( \mathbb E_{\Theta\sim\widehat L_M} -\mathbb E_{\Theta\sim L} \bigr) \bigl[ V_\ell^\star\bigl( \bar\Theta_h(\bar S_h,\pi_M(\bar X_h)), \bar\Theta_{h+1:h+\ell-1},\Theta \bigr) \bigr] \Bigr] \nonumber\\
&\quad+ \sum_{h=0}^{H-1}\gamma^{h+1}\, \mathbb E^{\pi_M}_x\Bigl[ \bigl( \mathbb E_{\Theta\sim\widehat L_M} -\mathbb E_{\Theta\sim L} \bigr) \bigl[ (\widehat V^\star_M-V_\ell^\star)\bigl( \bar\Theta_h(\bar S_h,\pi_M(\bar X_h)), \bar\Theta_{h+1:h+\ell-1},\Theta \bigr) \bigr] \Bigr] \nonumber\\
&\quad+ \gamma^H\,\mathbb E^{\pi_M}_x\bigl[ (\widehat V^\star_M-V_\ell^{\pi_M})(\bar X_H) \bigr].
\label{eq:lazy-policy-unrolled}
\end{align}
The last term is bounded by $\gamma^HV_{\max}$ in absolute value and vanishes as $H\to\infty$. For the first sum, $\pi_M$ is greedy with respect to $\widetilde Q_M$, so at every augmented state (see the proof of \eqref{eq:lazy-deployed-policy-loss} below for the two-line argument)
\begin{equation}
0\leq
\widehat V^\star_M(x)-\widehat Q^\star_M(x,\pi_M(x))
\leq
2\|\widetilde Q_M-\widehat Q^\star_M\|_\infty.
\label{eq:lazy-greedy-residual}
\end{equation}
For the third sum, at every augmented state,
\begin{equation*}
\Bigl| \bigl( \mathbb E_{\Theta\sim\widehat L_M} -\mathbb E_{\Theta\sim L} \bigr) \bigl[ (\widehat V^\star_M-V_\ell^\star) \bigl(\theta_0(s,\pi_M(x)),\theta_{1:\ell-1},\Theta\bigr) \bigr] \Bigr|
\leq \|\widehat L_M-L\|_1 \|\widehat V^\star_M-V_\ell^\star\|_\infty.
\end{equation*}
Combining these bounds with \eqref{eq:lazy-policy-value-split}, $\sum_{h\geq0}\gamma^h=V_{\max}$, and \eqref{eq:lazy-policy-unrolled} as $H\to\infty$ yields
\begin{align}
\|V_\ell^\star-V_\ell^{\pi_M}\|_\infty
&\leq
\underbrace{ \|V_\ell^\star-\widehat V^\star_M\|_\infty }_{\text{empirical value error}}
\nonumber\\
&\quad+
\underbrace{ \sup_x\Bigl| \sum_{h=0}^{\infty}\gamma^{h+1}\, \mathbb E^{\pi_M}_x\Bigl[ \bigl( \mathbb E_{\Theta\sim\widehat L_M} -\mathbb E_{\Theta\sim L} \bigr) \bigl[ V_\ell^\star\bigl( \bar\Theta_h(\bar S_h,\pi_M(\bar X_h)), \bar\Theta_{h+1:h+\ell-1},\Theta \bigr) \bigr] \Bigr] \Bigr| }_{\text{fixed-function transition error}}
\nonumber\\
&\quad+
\underbrace{ \gamma V_{\max}\|\widehat L_M-L\|_1 \|\widehat V^\star_M-V_\ell^\star\|_\infty }_{\text{data-dependent transition error}}
+
\underbrace{ 2V_{\max} \|\widetilde Q_M-\widehat Q^\star_M\|_\infty }_{\text{finite-planner error}}.
\label{eq:lazy-policy-loss-components}
\end{align}

We first control the fixed-function transition error. The difficulty is that $\widehat L_M$ is a product of empirical row distributions, not the empirical distribution of $M$ complete tables. The following lemma shows that, for a fixed function, its error nevertheless obeys a Bernstein inequality with the variance under $L$.

\begin{lemma}[Product-estimator Bernstein inequality]
\label{lem:lazy-product-bernstein}
For any deterministic $f:\Omega\to[0,V_{\max}]$, any $M\geq1$, and any $w>0$,
\begin{equation}
 \mathbb{P}\left(
 \bigl|\mathbb E_{\widehat L_M}[f]-\mathbb E_L[f]\bigr|>
 \sqrt{\frac{2\operatorname{Var}_L(f)\,w}{M}}+\frac{2V_{\max}w}{3M}
 \right)\leq2e^{-w}.
 \label{eq:lazy-product-bernstein}
\end{equation}
\end{lemma}

\begin{proof}
For each $(s,a)$, draw a uniform random permutation $\sigma_{s,a}$ of $\{0,\ldots,M-1\}$, independently across $(s,a)$ and independently of the observed tables; write $\mathbb E_\sigma$ for the expectation over these permutations. For $j\in\{0,\ldots,M-1\}$, form the table
\begin{equation*}
\Theta^{(j)}:=\bigl(\Theta_{\sigma_{s,a}(j)}(s,a)\bigr)_{(s,a)\in\mathcal S\times\mathcal A}\in\Omega.
\end{equation*}
Conditional on $\Theta_0,\ldots,\Theta_{M-1}$, for fixed $j$ the entries of $\Theta^{(j)}$ are independent across $(s,a)$, each with distribution $\widehat P_M(\cdot\mid s,a)$; hence $\Theta^{(j)}$ has conditional law $\widehat L_M$, and averaging over $j$ gives
\begin{equation}
    \mathbb E_\sigma\!\left[ \frac1M\sum_{j=0}^{M-1}  f\bigl(\Theta^{(j)}\bigr) \,\middle|\,   \Theta_0,\ldots,\Theta_{M-1} \right] = \mathbb E_{\widehat L_M}[f].
\label{eq:lazy-permutation-identity}
\end{equation}
Now fix the permutations. For each $(s,a)$ the indices $\sigma_{s,a}(0),\ldots,\sigma_{s,a}(M-1)$ are distinct, so the tables $\Theta^{(0)},\ldots,\Theta^{(M-1)}$ use distinct observations from every row. Independence across dates and rows implies that, for fixed permutations, $\Theta^{(0)},\ldots,\Theta^{(M-1)}$ are i.i.d.\ with law $L$. Conditional Jensen's inequality applied to \eqref{eq:lazy-permutation-identity} gives, for every $\lambda\in\mathbb R$,
\begin{align*}
\mathbb E\exp\!\bigl( \lambda\bigl[ \mathbb E_{\widehat L_M}[f]-\mathbb E_L[f] \bigr] \bigr)
&\leq
\mathbb E\,\mathbb E_\sigma \exp\!\left(  \frac{\lambda}{M}\sum_{j=0}^{M-1} \bigl[f\bigl(\Theta^{(j)}\bigr) -\mathbb E_L[f]\bigr] \right)
\\&=
\mathbb E\exp\!\left(  \frac{\lambda}{M} \sum_{j=0}^{M-1}  \bigl[f(\Theta_j)-\mathbb E_L[f]\bigr] \right).
\end{align*}
Thus the moment-generating function of the product-estimator error is bounded by that of the average of $M$ i.i.d.\ copies of $f(\Theta)-\mathbb E_L[f]$, which take values in an interval of length $V_{\max}$. Bernstein's inequality applied to the positive and negative tails proves \eqref{eq:lazy-product-bernstein}.
\end{proof}

\begin{lemma}[Model confidence bounds]
\label{lem:lazy-confidence}
There is an event $\mathcal E_{\mathrm{model}}$ of probability at least $1-3\delta/16$ on which the following three inequalities hold simultaneously for every dyadic update $M<T$.
\begin{enumerate}
\item For every $s\in\mathcal S$, $a\in\mathcal A$, and $\theta_{0:\ell-1}\in\Omega^\ell$,
\begin{align}
&\Bigl| \mathbb E_{\Theta\sim\widehat L_M} \bigl[ V_\ell^\star\bigl( \theta_0(s,a),\theta_{1:\ell-1},\Theta \bigr) \bigr] - \mathbb E_{\Theta\sim L} \bigl[ V_\ell^\star\bigl( \theta_0(s,a),\theta_{1:\ell-1},\Theta \bigr) \bigr] \Bigr|
\nonumber\\
&\qquad\leq \sqrt{ \frac{ 2\Lambda\operatorname{Var}_{\Theta\sim L} \bigl( V_\ell^\star( \theta_0(s,a),\theta_{1:\ell-1},\Theta ) \bigr) }{M} } + \frac{2V_{\max}\Lambda}{3M}.
\label{eq:lazy-fixed-confidence}
\end{align}
\item The laws of transition tables satisfy
\begin{equation}
\|\widehat L_M-L\|_1 \leq 2\sqrt{\frac{\Lambda_0}{M}}.
\label{eq:lazy-l1-confidence}
\end{equation}
\item Every bounded $g:\mathcal S\times\Omega^\ell\to\mathbb R$, including one chosen after observing the tables, satisfies
\begin{equation}
\sup_{s\in\mathcal S,\ a\in\mathcal A,\ \theta_{0:\ell-1}\in\Omega^\ell}
\Bigl|
\bigl(\mathbb E_{\Theta\sim\widehat L_M}-\mathbb E_{\Theta\sim L}\bigr)
\bigl[
g\bigl(\theta_0(s,a),\theta_{1:\ell-1},\Theta\bigr)
\bigr]
\Bigr|
\leq
2\sqrt{\frac{\Lambda_0}{M}}\,\|g\|_\infty.
\label{eq:lazy-uniform-confidence}
\end{equation}
\end{enumerate}
\end{lemma}

\begin{proof}
\emph{(i).} For fixed $u\in\mathcal S$ and $\theta_{1:\ell-1}\in\Omega^{\ell-1}$, the function $\Theta\mapsto V_\ell^\star(u,\theta_{1:\ell-1},\Theta)$ is deterministic and does not depend on the observed tables. Apply Lemma~\ref{lem:lazy-product-bernstein} to it with $w=\Lambda$. There are at most $|\mathcal S||\Omega|^{\ell-1}$ such functions and at most $J$ update sizes, so a union bound gives failure probability at most $2e^{-\Lambda}\,|\mathcal S||\Omega|^{\ell-1}J=\delta/8$. Taking $u=\theta_0(s,a)$ yields \eqref{eq:lazy-fixed-confidence}.

\emph{(ii).} By \citet[Theorem~I.3]{Agrawal_2020}, applied with exponential parameter $M/2$, each empirical transition row satisfies
\begin{equation*}
\mathbb E\exp\!\Bigl[
\tfrac M2\operatorname{KL}\bigl(
\widehat P_M(\cdot\mid s,a)\,\big\Vert\, P(\cdot\mid s,a)
\bigr)
\Bigr]
\leq 2^{|\mathcal S|-1}.
\end{equation*}
The row estimates are independent and the KL divergence is additive over product distributions, hence
\begin{equation*}
\mathbb E\exp\!\Bigl[
\tfrac M2\operatorname{KL}(\widehat L_M\Vert L)
\Bigr]
\leq 2^{|\mathcal S||\mathcal A|(|\mathcal S|-1)}.
\end{equation*}
Markov's inequality gives
\begin{equation*}
\mathbb P\Bigl(
\operatorname{KL}(\widehat L_M\Vert L)
>
\tfrac{2\Lambda_0}{M}
\Bigr)
\leq
2^{|\mathcal S||\mathcal A|(|\mathcal S|-1)}e^{-\Lambda_0}
=\frac{\delta}{16J}.
\end{equation*}
A union bound over the at most $J$ updates shows that this KL bound holds at every update with probability at least $1-\delta/16$, and Pinsker's inequality then yields $\|\widehat L_M-L\|_1\leq\sqrt{2\operatorname{KL}(\widehat L_M\Vert L)}\leq2\sqrt{\Lambda_0/M}$.

\emph{(iii).} Fix the observed tables, a bounded $g$, and $(s,a,\theta_{0:\ell-1})$. The difference of expectations in \eqref{eq:lazy-uniform-confidence} equals $\sum_{\Theta\in\Omega}\bigl(\widehat L_M(\Theta)-L(\Theta)\bigr)g(\theta_0(s,a),\theta_{1:\ell-1},\Theta)$, whose absolute value is at most $\|\widehat L_M-L\|_1\|g\|_\infty$. This holds for all $g$ and all $(s,a,\theta_{0:\ell-1})$ simultaneously, so (iii) is a deterministic consequence of (ii).

The total failure probability is $\delta/8+\delta/16=3\delta/16$.
\end{proof}

The fixed-function term in \eqref{eq:lazy-policy-loss-components} accumulates the one-step errors \eqref{eq:lazy-fixed-confidence} along a discounted trajectory. To retain the variance dependence of the Bernstein bound, we control the discounted sum of conditional standard deviations of $V_\ell^\star$.

\begin{lemma}[Discounted variance bound]
\label{lem:lazy-resolvent}
Let $\pi$ be a deterministic stationary policy. For every augmented state $x$,
\begin{equation}
\sum_{h=0}^{\infty}\gamma^{h+1}\,
\mathbb E^\pi_x\!\left[
\sqrt{
\operatorname{Var}\bigl(
V_\ell^\star(\bar X_{h+1})\mid \bar X_h
\bigr)
}
\right]
\leq
\sqrt{2V_{\max}^3},
\label{eq:lazy-resolvent-variance}
\end{equation}
where the conditional variance is over the appended table $\bar\Theta_{h+\ell}\sim L$ given $\bar X_h$. Consequently, on $\mathcal E_{\mathrm{model}}$, every deterministic stationary policy $\pi$, including one selected using the observed tables, satisfies
\begin{align}
&\sup_x\Bigl|
\sum_{h=0}^{\infty}\gamma^{h+1}\,
\mathbb E^\pi_x\Bigl[
\bigl(\mathbb E_{\Theta\sim\widehat L_M}-\mathbb E_{\Theta\sim L}\bigr)
\bigl[
V_\ell^\star\bigl(
\bar\Theta_h(\bar S_h,\pi(\bar X_h)),
\bar\Theta_{h+1:h+\ell-1},\Theta
\bigr)
\bigr]
\Bigr]
\Bigr|
\\& \leq
2\sqrt{\frac{V_{\max}^3\Lambda}{M}}
+
\frac{2V_{\max}^2\Lambda}{3M}.
\label{eq:lazy-resolvent-error}
\end{align}
For $\ell=1$, the list $\bar\Theta_{h+1:h+\ell-1}$ is empty.
\end{lemma}

\begin{proof}
\emph{Proof of \eqref{eq:lazy-resolvent-variance}.}
Fix $x$. Bellman optimality and nonnegative rewards imply, for every augmented state $y$,
\begin{equation}
V_\ell^\star(y)-\gamma\,\mathbb E^\pi_y\bigl[V_\ell^\star(\bar X_1)\bigr]
\geq r(y,\pi(y))\geq0.
\label{eq:lazy-nonnegative-residual}
\end{equation}
Also $0\leq V_\ell^\star\leq V_{\max}$. Expanding the conditional variance,
\begin{align}
\gamma^2\operatorname{Var}\bigl(V_\ell^\star(\bar X_1)\mid \bar X_0=y\bigr)
&=
\gamma^2\mathbb E^\pi_y\bigl[V_\ell^\star(\bar X_1)^2\bigr]
-\bigl(\gamma\,\mathbb E^\pi_y[V_\ell^\star(\bar X_1)]\bigr)^2
\nonumber\\
&\leq
\gamma\,\mathbb E^\pi_y\bigl[V_\ell^\star(\bar X_1)^2\bigr]
-V_\ell^\star(y)^2
+2V_{\max}\bigl(V_\ell^\star(y)-\gamma\,\mathbb E^\pi_y[V_\ell^\star(\bar X_1)]\bigr).
\label{eq:lazy-one-step-variance}
\end{align}
Here $\gamma^2\leq\gamma$ bounds the first term; for the squared mean, write $\gamma\,\mathbb E^\pi_y[V_\ell^\star(\bar X_1)]$ as $V_\ell^\star(y)$ minus the nonnegative quantity in \eqref{eq:lazy-nonnegative-residual}, expand the square, and drop its nonpositive quadratic term.

Apply \eqref{eq:lazy-one-step-variance} at $y=\bar X_h$, multiply by $\gamma^h$, take $\mathbb E^\pi_x$, and sum over $h\geq0$. Both terms on the right telescope, and the terminal terms vanish because $V_\ell^\star$ is bounded and $\gamma<1$:
\begin{equation*}
\sum_{h=0}^{\infty}\gamma^h\,
\mathbb E^\pi_x\bigl[
\gamma^2\operatorname{Var}\bigl(V_\ell^\star(\bar X_{h+1})\mid \bar X_h\bigr)
\bigr]
\leq
-V_\ell^\star(x)^2+2V_{\max}V_\ell^\star(x)
\leq 2V_{\max}^2.
\end{equation*}
Jensen's inequality ($\mathbb E\sqrt{\cdot}\leq\sqrt{\mathbb E\,\cdot}$), Cauchy--Schwarz, and $\sum_{h\geq0}\gamma^h=V_{\max}$ give
\begin{align*}
&\sum_{h=0}^{\infty}\gamma^{h+1}\,
\mathbb E^\pi_x\Bigl[\sqrt{\operatorname{Var}\bigl(V_\ell^\star(\bar X_{h+1})\mid \bar X_h\bigr)}\Bigr]
\\&\leq
\sqrt{
\Bigl(\sum_{h=0}^{\infty}\gamma^h\Bigr)
\Bigl(\sum_{h=0}^{\infty}\gamma^h\,\mathbb E^\pi_x\bigl[\gamma^2\operatorname{Var}\bigl(V_\ell^\star(\bar X_{h+1})\mid \bar X_h\bigr)\bigr]\Bigr)
}\\&
\leq\sqrt{2V_{\max}^3}.
\end{align*}

\emph{Proof of \eqref{eq:lazy-resolvent-error}.}
Conditional on $\bar X_h$, the successor $\bar\Theta_h(\bar S_h,\pi(\bar X_h))$ and the tables $\bar\Theta_{h+1:h+\ell-1}$ are fixed; only the appended table is averaged over $L$ or $\widehat L_M$. Hence \eqref{eq:lazy-fixed-confidence} bounds the difference of the two expectations by
\begin{equation*}
\sqrt{\frac{2\Lambda}{M}}\,
\sqrt{\operatorname{Var}\bigl(V_\ell^\star(\bar X_{h+1})\mid \bar X_h\bigr)}
+\frac{2V_{\max}\Lambda}{3M}.
\end{equation*}
Multiply by $\gamma^{h+1}$, take $\mathbb E^\pi_x$, and sum over $h$. The variance terms are bounded by \eqref{eq:lazy-resolvent-variance}, which gives $2\sqrt{V_{\max}^3\Lambda/M}$; the constant terms sum to at most $\frac{2V_{\max}\Lambda}{3M}\sum_{h\geq0}\gamma^{h+1}\leq\frac{2V_{\max}^2\Lambda}{3M}$. The argument applies to a data-dependent policy because \eqref{eq:lazy-fixed-confidence} holds simultaneously at every augmented state and action; no union bound over policies is needed.
\end{proof}

The remaining statistical quantity in \eqref{eq:lazy-policy-loss-components} is $\|\widehat V^\star_M-V_\ell^\star\|_\infty$. Comparing the empirical and true Bellman optimality equations yields an inequality in which this quantity appears on both sides, with coefficient $\gamma V_{\max}\|\widehat L_M-L\|_1$ on the right. For $M\geq M_0$ this coefficient is at most $1/2$ and the term can be absorbed. The next lemma makes this explicit and also bounds the loss of any policy with a small empirical Bellman residual.

\begin{lemma}[Stability]
\label{lem:lazy-policy-stability}
Suppose $M\geq M_0$. On $\mathcal E_{\mathrm{model}}$,
\begin{equation}
\|\widehat V^\star_M-V_\ell^\star\|_\infty
\leq 4\sqrt{\frac{V_{\max}^3\Lambda}{M}}
+\frac{4V_{\max}^2\Lambda}{3M}.
\label{eq:lazy-value-stability}
\end{equation}
Moreover, if a deterministic stationary policy $\pi$ has empirical Bellman residual at most $b\geq0$, i.e.\ for every $x=(s,\theta_{0:\ell-1})$,
\begin{equation}
0\leq
\widehat V^\star_M(x)-r(s,\pi(x))
-\gamma\,\mathbb E_{\Theta\sim\widehat L_M}
\bigl[
\widehat V^\star_M\bigl(
\theta_0(s,\pi(x)),\theta_{1:\ell-1},\Theta
\bigr)
\bigr]
\leq b,
\label{eq:lazy-empirical-residual}
\end{equation}
then its true value satisfies
\begin{equation}
\|V_\ell^\star-V_\ell^\pi\|_\infty
\leq 8\sqrt{\frac{V_{\max}^3\Lambda}{M}}
+\frac{8V_{\max}^2\Lambda}{3M}
+V_{\max}b.
\label{eq:lazy-policy-stability}
\end{equation}
\end{lemma}

\begin{proof}
By \eqref{eq:lazy-l1-confidence} and $M\geq M_0\geq16V_{\max}^2\Lambda_0$,
\begin{equation}
\gamma V_{\max}\|\widehat L_M-L\|_1\leq 2V_{\max}\sqrt{\Lambda_0/M}\leq\tfrac12.
\label{eq:lazy-absorption-coefficient}
\end{equation}

\emph{Proof of \eqref{eq:lazy-value-stability}.}
Let $\widehat\pi^\star_M$ be an optimal policy for the empirical model. Fix $x=(s,\theta_{0:\ell-1})$. The empirical Bellman equation for $\widehat\pi^\star_M$ and the true Bellman inequality for the same action give
\begin{align*}
\widehat V^\star_M(x)-V_\ell^\star(x)
&\leq
\gamma\,\mathbb E_{\Theta\sim L}
\bigl[
(\widehat V^\star_M-V_\ell^\star)
\bigl(\theta_0(s,\widehat\pi^\star_M(x)),\theta_{1:\ell-1},\Theta\bigr)
\bigr]\\
&\quad+
\gamma\bigl(\mathbb E_{\Theta\sim\widehat L_M}-\mathbb E_{\Theta\sim L}\bigr)
\bigl[
V_\ell^\star\bigl(\theta_0(s,\widehat\pi^\star_M(x)),\theta_{1:\ell-1},\Theta\bigr)
\bigr]\\
&\quad+
\gamma\bigl(\mathbb E_{\Theta\sim\widehat L_M}-\mathbb E_{\Theta\sim L}\bigr)
\bigl[
(\widehat V^\star_M-V_\ell^\star)
\bigl(\theta_0(s,\widehat\pi^\star_M(x)),\theta_{1:\ell-1},\Theta\bigr)
\bigr].
\end{align*}
Iterate this inequality along the auxiliary chain under $\mathbb E^{\widehat\pi^\star_M}_x$. The discounted sum of the second term is bounded by \eqref{eq:lazy-resolvent-error}. For the third term, \eqref{eq:lazy-uniform-confidence} gives at every state the bound $\|\widehat L_M-L\|_1\|\widehat V^\star_M-V_\ell^\star\|_\infty$, whose discounted sum is at most $\gamma V_{\max}\|\widehat L_M-L\|_1\|\widehat V^\star_M-V_\ell^\star\|_\infty$.

For the opposite difference $V_\ell^\star-\widehat V^\star_M$, repeat the argument with a true optimal policy $\pi^\star_\ell$: the true Bellman equation holds for this policy and the empirical Bellman optimality inequality holds for its action; the terms involving $V_\ell^\star$ change sign but their absolute values obey the same bounds. Combining the two directions,
\begin{equation}
\|\widehat V^\star_M-V_\ell^\star\|_\infty
\leq
2\sqrt{\frac{V_{\max}^3\Lambda}{M}}
+\frac{2V_{\max}^2\Lambda}{3M}
+\gamma V_{\max}\|\widehat L_M-L\|_1\|\widehat V^\star_M-V_\ell^\star\|_\infty.
\label{eq:lazy-stability-absorption}
\end{equation}
By \eqref{eq:lazy-absorption-coefficient}, the last term is at most $\tfrac12\|\widehat V^\star_M-V_\ell^\star\|_\infty$; moving it to the left proves \eqref{eq:lazy-value-stability}.

\emph{Proof of \eqref{eq:lazy-policy-stability}.}
Subtract the true Bellman equation of $\pi$ from \eqref{eq:lazy-empirical-residual} and iterate along the auxiliary chain under $\mathbb E^\pi_x$. The discounted sum of the empirical residual is at most $V_{\max}b$. Splitting $\widehat V^\star_M=V_\ell^\star+(\widehat V^\star_M-V_\ell^\star)$ in the remaining model error gives two terms: by \eqref{eq:lazy-resolvent-error}, the part involving $V_\ell^\star$ contributes at most $2\sqrt{V_{\max}^3\Lambda/M}+2V_{\max}^2\Lambda/(3M)$; by \eqref{eq:lazy-uniform-confidence} and \eqref{eq:lazy-absorption-coefficient}, the other part contributes at most $\tfrac12\|\widehat V^\star_M-V_\ell^\star\|_\infty$. Hence
\begin{equation*}
\|\widehat V^\star_M-V_\ell^\pi\|_\infty
\leq
V_{\max}b
+2\sqrt{\frac{V_{\max}^3\Lambda}{M}}
+\frac{2V_{\max}^2\Lambda}{3M}
+\frac12\|\widehat V^\star_M-V_\ell^\star\|_\infty.
\end{equation*}
Finally, 

\begin{align*}
    &\|V_\ell^\star-V_\ell^\pi\|_\infty\\&\leq\|V_\ell^\star-\widehat V^\star_M\|_\infty+\|\widehat V^\star_M-V_\ell^\pi\|_\infty\\& \leq\tfrac32\|\widehat V^\star_M-V_\ell^\star\|_\infty+V_{\max}b+2\sqrt{V_{\max}^3\Lambda/M}+2V_{\max}^2\Lambda/(3M)
\end{align*}, and \eqref{eq:lazy-value-stability} gives \eqref{eq:lazy-policy-stability}.
\end{proof}

It remains to control the finite-planner error $2V_{\max}\|\widetilde Q_M-\widehat Q^\star_M\|_\infty$ in \eqref{eq:lazy-policy-loss-components}. Conditionally on the observed tables, $\widehat V^\star_M$ is fixed and the dictionary entries $Z^1,\ldots,Z^{N_M}$ are independent draws from $\widehat L_M$. We bound the dictionary sampling error by conditional concentration, then account for the finite number of value-iteration steps, exactly as in the known-model analysis of Appendix~\ref{proof:main-RPTAS}.

\begin{lemma}[Uniform action-score accuracy]
\label{lem:lazy-dictionary}
There is an event $\mathcal E_{\mathrm{dict}}$ of probability at least $1-\delta/8$ on which, at every dyadic update $M<T$,
\begin{equation}
\|\widetilde Q_M-\widehat Q^\star_M\|_\infty
\leq
\sqrt{\frac{V_{\max}\Lambda}{2M}}
+M^{-2},
\label{eq:lazy-dictionary-accuracy}
\end{equation}
where the norm ranges over all augmented state--action pairs, including windows outside the sampled dictionary. Consequently, on $\mathcal E_{\mathrm{model}}\cap\mathcal E_{\mathrm{dict}}$, every update $M\geq M_0$ satisfies
\begin{equation}
\|V_\ell^\star-V_\ell^{\pi_M}\|_\infty
\leq
8\sqrt{\frac{V_{\max}^3\Lambda}{M}}
+\frac{8V_{\max}^2\Lambda}{3M}
+2V_{\max}\left[
\sqrt{\frac{V_{\max}\Lambda}{2M}}
+M^{-2}
\right].
\label{eq:lazy-deployed-policy-loss}
\end{equation}
\end{lemma}

\begin{proof}
\emph{Proof of \eqref{eq:lazy-dictionary-accuracy}.}
Fix an update and condition on the $M$ observed tables; then $\widehat L_M$ and $\widehat V^\star_M$ are fixed and $Z^1,\ldots,Z^{N_M}$ are i.i.d.\ with law $\widehat L_M$. For each $u\in\mathcal S$ and $\theta_{1:\ell-1}\in\Omega^{\ell-1}$, Hoeffding's inequality applies to the function $Z\mapsto\widehat V^\star_M(u,\theta_{1:\ell-1},Z)\in[0,V_{\max}]$. There are at most $|\mathcal S||\Omega|^{\ell-1}$ such functions, so a union bound shows that, except with conditional probability at most $\delta/(8J)$,
\begin{equation}
\sup_{u,\theta_{1:\ell-1}}
\Bigl|
\frac{1}{N_M}\sum_{j=1}^{N_M}
\widehat V^\star_M(u,\theta_{1:\ell-1},Z^j)
-
\mathbb E_{\Theta\sim\widehat L_M}\bigl[\widehat V^\star_M(u,\theta_{1:\ell-1},\Theta)\bigr]
\Bigr|
\leq
V_{\max}\sqrt{\frac{\Lambda}{2N_M}}.
\label{eq:lazy-dictionary-hoeffding}
\end{equation}
For this proof only, let $V^{\mathrm{dict}}_M$ and $Q^{\mathrm{dict}}_M$ be the exact optimal value and action scores of the augmented MDP when each incoming table is drawn uniformly from $\{Z^1,\ldots,Z^{N_M}\}$. The Bellman operators for this law and for $\widehat L_M$ are $\gamma$-contractions; comparing them at $\widehat V^\star_M$ using \eqref{eq:lazy-dictionary-hoeffding} gives
\begin{align*}
\|V^{\mathrm{dict}}_M-\widehat V^\star_M\|_\infty
&\leq
\gamma V_{\max}^2\sqrt{\frac{\Lambda}{2N_M}},\\
\|Q^{\mathrm{dict}}_M-\widehat Q^\star_M\|_\infty
&\leq
\gamma\|V^{\mathrm{dict}}_M-\widehat V^\star_M\|_\infty
+\gamma V_{\max}\sqrt{\frac{\Lambda}{2N_M}}
\\& \leq
\gamma V_{\max}^2\sqrt{\frac{\Lambda}{2N_M}},
\end{align*}
where the last step uses $1+\gamma V_{\max}=V_{\max}$.

On dictionary windows, the Bellman equation of $V^{\mathrm{dict}}_M$ is the value-iteration update \eqref{eq:vi-finite-statespace} with dictionary size $N_M$ and tables $Z^1,\ldots,Z^{N_M}$; thus $K_M$ iterations from zero approximate its restriction with error at most $\gamma^{K_M}V_{\max}$. For an arbitrary queried window, after $\ell$ transitions every table in the window belongs to the dictionary, so the backward recursion \eqref{eq:extension-recursion} and the score formula \eqref{eq:approx-Q-extension}, run with dictionary $Z^1,\ldots,Z^{N_M}$, return $Q^{\mathrm{dict}}_M$ exactly when initialized with the exact dictionary values (Lemma~\ref{lem:extension-identity}), and each backward step multiplies a terminal error by at most $\gamma$ (Lemma~\ref{lem:extension-Lipschitz}). Consequently, uniformly over all queried windows and actions,
\begin{equation*}
\|\widetilde Q_M-Q^{\mathrm{dict}}_M\|_\infty \leq\gamma^{K_M+\ell}V_{\max}\leq e^{-(1-\gamma)K_M}V_{\max}\leq M^{-2},
\end{equation*}
by the choice $K_M=\lceil V_{\max}\log(V_{\max}M^2)\rceil$. Since $N_M\geq MV_{\max}^3$,
\begin{equation*}
\gamma V_{\max}^2\sqrt{\frac{\Lambda}{2N_M}}\leq\sqrt{\frac{V_{\max}\Lambda}{2M}}.
\end{equation*}
The triangle inequality proves \eqref{eq:lazy-dictionary-accuracy} at this update, conditionally on the observed tables. Removing the conditioning and taking a union bound over at most $J$ updates gives the stated probability.

\emph{Proof of \eqref{eq:lazy-deployed-policy-loss}.}
Fix $x$ and let $a^\star$ maximize $\widehat Q^\star_M(x,\cdot)$. Since $\pi_M(x)$ maximizes $\widetilde Q_M(x,\cdot)$,
\begin{align*}
0\leq
\widehat V^\star_M(x)-\widehat Q^\star_M(x,\pi_M(x))
&\leq
\bigl[\widehat Q^\star_M(x,a^\star)-\widetilde Q_M(x,a^\star)\bigr]
+\bigl[\widetilde Q_M(x,\pi_M(x))-\widehat Q^\star_M(x,\pi_M(x))\bigr]\\
&\leq 2\|\widetilde Q_M-\widehat Q^\star_M\|_\infty,
\end{align*}
which is \eqref{eq:lazy-greedy-residual}, and by \eqref{eq:lazy-dictionary-accuracy} this is at most $2\bigl[\sqrt{V_{\max}\Lambda/(2M)}+M^{-2}\bigr]$. This is the empirical Bellman residual of $\pi_M$ at $x$; it holds at every augmented state, so \eqref{eq:lazy-policy-stability} with this residual $b$ gives \eqref{eq:lazy-deployed-policy-loss}.
\end{proof}

The Bellman equation of the active policy decomposes each summand $(1-\gamma)W_t(X_t)-r_t$ into a difference of successive values and a martingale difference. Since $W_t$ is fixed within each epoch, the value differences telescope; the following lemma controls the epoch boundaries and the martingale fluctuations.

\begin{lemma}[Epochwise value-to-reward comparison]
\label{lem:lazy-epoch-martingale}
Suppose $\{0,\ldots,T-1\}$ is partitioned into at most $J$ deterministic consecutive epochs. At the start of each epoch, a deterministic stationary policy is selected using the information then available (possibly depending on previous observations and on the algorithm's randomization) and executed throughout the epoch. Let $W_t$ be the true infinite-horizon value of the policy active at time $t$. For every $w>0$, with probability at least $1-2e^{-w}$,
\begin{equation}
\sum_{t=0}^{T-1}
\bigl[(1-\gamma)W_t(X_t)-r_t\bigr]
\leq
V_{\max}J
+\sqrt{2\bigl(4V_{\max}T+2V_{\max}^2J+16V_{\max}^2w\bigr)w}
+\tfrac{2}{3}V_{\max}w.
\label{eq:lazy-epoch-bound}
\end{equation}
\end{lemma}

\begin{proof}
Let $\mathcal F_t$ be the $\sigma$-field of the information available after choosing $A_t$ and before revealing the new table $\Theta_{t+\ell}$; in particular $W_t$ and $X_t$ are $\mathcal F_t$-measurable and $S_{t+1}=\Theta_t(S_t,A_t)$ is already determined. Since the active policy is executed at time $t$, its Bellman equation gives $W_t(X_t)=r_t+\gamma\,\mathbb E[W_t(X_{t+1})\mid\mathcal F_t]$, hence
\begin{align}
(1-\gamma)W_t(X_t)-r_t
&=
\gamma\bigl[W_t(X_{t+1})-W_t(X_t)\bigr]
+\gamma\bigl[\mathbb E[W_t(X_{t+1})\mid\mathcal F_t]-W_t(X_{t+1})\bigr].
\label{eq:lazy-epoch-decomposition}
\end{align}
The first bracket telescopes within each epoch, because $W_t$ is the same function throughout the epoch; as $0\leq W_t\leq V_{\max}$, its total contribution is at most $V_{\max}J$.

The second bracket defines martingale differences with conditional mean zero and absolute value at most $V_{\max}$. Their total conditional variance is
\begin{equation}
\mathcal V_T:=\sum_{t=0}^{T-1}\gamma^2\operatorname{Var}\bigl(W_t(X_{t+1})\mid\mathcal F_t\bigr).
\label{eq:lazy-predictable-variance}
\end{equation}
We first bound $\mathcal V_T$. Using the Bellman equation, $0\leq r_t\leq1$ and $0\leq W_t\leq V_{\max}$,
\begin{align*}
\gamma^2\operatorname{Var}\bigl(W_t(X_{t+1})\mid\mathcal F_t\bigr)
&=
\gamma^2\mathbb E[W_t(X_{t+1})^2\mid\mathcal F_t]-\bigl(W_t(X_t)-r_t\bigr)^2\\
&\leq
\gamma^2\bigl[W_t(X_{t+1})^2-W_t(X_t)^2\bigr]
+2V_{\max}
\\& +\gamma^2\bigl[\mathbb E[W_t(X_{t+1})^2\mid\mathcal F_t]-W_t(X_{t+1})^2\bigr].
\end{align*}
The squared-value differences telescope within each epoch and contribute at most $V_{\max}^2J$, so
\begin{equation}
\mathcal V_T
\leq
2V_{\max}T+V_{\max}^2J
+\sum_{t=0}^{T-1}\gamma^2\bigl[\mathbb E[W_t(X_{t+1})^2\mid\mathcal F_t]-W_t(X_{t+1})^2\bigr].
\label{eq:lazy-variance-recursion}
\end{equation}
The last sum is a martingale with increments bounded by $V_{\max}^2$; since $z\mapsto z^2$ is $2V_{\max}$-Lipschitz on $[0,V_{\max}]$, its total conditional variance is at most $4V_{\max}^2\mathcal V_T$. Applying $e^z\leq1+z+z^2$ for $|z|\leq1$ to the increments divided by $8V_{\max}^2$, followed by the exponential supermartingale inequality, gives
\begin{equation*}
\mathbb P\left\{
\sum_{t=0}^{T-1}\gamma^2\bigl[\mathbb E[W_t(X_{t+1})^2\mid\mathcal F_t]-W_t(X_{t+1})^2\bigr]
>\frac{\mathcal V_T}{2}+8V_{\max}^2w
\right\}
\leq e^{-w}.
\end{equation*}
Combined with \eqref{eq:lazy-variance-recursion}, this shows that, except with probability $e^{-w}$,
\begin{equation}
\mathcal V_T\leq4V_{\max}T+2V_{\max}^2J+16V_{\max}^2w.
\label{eq:lazy-variance-total}
\end{equation}
Finally, Freedman's inequality for martingales with increments bounded by $V_{\max}$ shows that, on the event \eqref{eq:lazy-variance-total}, the sum of the second brackets in \eqref{eq:lazy-epoch-decomposition} is at most $\sqrt{2(4V_{\max}T+2V_{\max}^2J+16V_{\max}^2w)w}+\tfrac23V_{\max}w$ except with probability $e^{-w}$. Adding the telescoping contribution proves \eqref{eq:lazy-epoch-bound}.
\end{proof}

We sum the policy losses in \eqref{eq:lazy-final-decomposition}. On $\mathcal E_{\mathrm{model}}\cap\mathcal E_{\mathrm{dict}}$, \eqref{eq:lazy-deployed-policy-loss} and $(1-\gamma)V_{\max}=1$ imply, whenever $M_t\geq M_0$,
\begin{equation}
(1-\gamma)\bigl[V_\ell^\star(X_t)-W_t(X_t)\bigr]
\leq
(8+\sqrt2)\sqrt{\frac{V_{\max}\Lambda}{M_t}}
+\frac{8V_{\max}\Lambda}{3M_t}
+\frac{2}{M_t^2}.
\label{eq:lazy-pointwise-policy-loss}
\end{equation}
At $t=0$ and at times with $M_t<M_0$, the normalized policy loss is at most $1$; since $M_t\geq t/2$, there are at most $2M_0$ such times. For $t\geq1$, $M_t\geq t/2$ and an epoch starting at $M$ contains at most $M$ steps, so
\begin{equation}
\sum_{t=1}^{T-1}\frac{1}{\sqrt{M_t}}\leq2\sqrt{2T},
\qquad
\sum_{t=1}^{T-1}\frac{1}{M_t}\leq J,
\qquad
\sum_{t=1}^{T-1}\frac{1}{M_t^2}\leq2.
\label{eq:lazy-dyadic-sums}
\end{equation}
Summing \eqref{eq:lazy-pointwise-policy-loss} therefore gives
\begin{equation}
(1-\gamma)\sum_{t=0}^{T-1}\bigl[V_\ell^\star(X_t)-W_t(X_t)\bigr]
\leq
2M_0+4
+(16\sqrt2+4)\sqrt{V_{\max}T\Lambda}
+\frac{8}{3}V_{\max}J\Lambda.
\label{eq:lazy-policy-loss-sum}
\end{equation}

For the value-to-reward error, the epochs are deterministic and the policy is fixed within each of them, so Lemma~\ref{lem:lazy-epoch-martingale} applies. Taking $w=\log(16/\delta)$ gives, with probability at least $1-\delta/8$,
\begin{equation}
\sum_{t=0}^{T-1}\bigl[(1-\gamma)W_t(X_t)-r_t\bigr]
\leq
V_{\max}J
+\sqrt{2\Bigl(4V_{\max}T+2V_{\max}^2J+16V_{\max}^2\log\tfrac{16}{\delta}\Bigr)\log\tfrac{16}{\delta}}
+\frac{2}{3}V_{\max}\log\frac{16}{\delta}.
\label{eq:lazy-value-to-reward-sum}
\end{equation}
This bound holds under the original probability law, without conditioning on $\mathcal E_{\mathrm{model}}$ or $\mathcal E_{\mathrm{dict}}$.

Substituting \eqref{eq:lazy-policy-loss-sum} and \eqref{eq:lazy-value-to-reward-sum} into \eqref{eq:lazy-final-decomposition}, and using $\sqrt{2(a+b+c)w}\leq\sqrt{2aw}+\sqrt{2bw}+\sqrt{2cw}$, we obtain on an event of probability at least
\begin{equation*}
1-\Bigl(\frac{3\delta}{16}+\frac{\delta}{8}+\frac{\delta}{8}\Bigr)=1-\frac{7\delta}{16}\geq1-\delta
\end{equation*}
the bound
\begin{align*}
\operatorname{Reg}(T)
&\leq
2M_0+4
+(16\sqrt2+4)\sqrt{V_{\max}T\Lambda}
+\frac{8}{3}V_{\max}J\Lambda
+V_{\max}J\\
&\quad+\sqrt{8V_{\max}T\log\tfrac{16}{\delta}}
+2V_{\max}\sqrt{J\log\tfrac{16}{\delta}}
+\bigl(\sqrt{32}+\tfrac23\bigr)V_{\max}\log\tfrac{16}{\delta}.
\end{align*}
Recalling $\Lambda=\log(16J/\delta)+\bigl(1+(\ell-1)|\mathcal S||\mathcal A|\bigr)\log|\mathcal S|$, $M_0=\lceil16V_{\max}^2\Lambda_0\rceil$ with $\Lambda_0\leq|\mathcal S|^2|\mathcal A|\log2+\log(16J/\delta)$, $J=O(\log(T+1))$, and $V_{\max}=(1-\gamma)^{-1}$, this gives, for fixed $\ell$,
\begin{equation*}
\operatorname{Reg}(T)
\leq
\widetilde O_\ell\!\left(
\sqrt{\frac{|\mathcal S||\mathcal A|\,T}{1-\gamma}}
+\frac{|\mathcal S|^2|\mathcal A|}{(1-\gamma)^2}
\right),
\end{equation*}
which is \eqref{thm:learning}. 
This proves the theorem. \hfill$\square$

\begin{remark}[Conversion to the undiscounted return deficit]
\label{rem:lazy-deficit-conversion}
Let $R_t:=\sum_{k\geq0}\gamma^kr_{t+k}$ denote the realised discounted return from time $t$, and $\mathcal R_\ell(T):=\sum_{t=0}^{T-1}[V_\ell^\star(X_t)-R_t]$. The identity $r_t=R_t-\gamma R_{t+1}$ gives $\operatorname{Reg}(T)=(1-\gamma)\mathcal R_\ell(T)+\gamma(R_T-R_0)$, and since $0\leq R_t\leq V_{\max}$,
\begin{equation}
\mathcal R_\ell(T)\leq V_{\max}\operatorname{Reg}(T)+\gamma V_{\max}^2.
\label{eq:lazy-deficit-conversion}
\end{equation}
\end{remark}

\section{Experimental Protocol}
\label{app:wind-experiments}
We evaluate RPTAS on the wind-farm storage-control benchmark of \citet{lu2025reinforcementlearningimperfecttransition}. We use the wind, price, and mismatch arrays produced by their preprocessing, together with the same state and action discretisations, battery dynamics, and evaluation interval. Code is available \href{https://anonymous.4open.science/r/Look-ahead-C795/README.md}{here}

\paragraph{Control objective.}
At each time-step, the wind farm announces how much electricity it will supply to the grid. If its realised production is lower than the announced amount, it pays a penalty for the resulting shortfall. The battery can store energy when production is abundant and release it to compensate for future shortages. Since its capacity is limited, the agent must decide when to store or use energy based on the predicted prices and wind mismatches. Its objective is to minimise the expected discounted sum of imbalance penalties.The planning problem is therefore to allocate the limited battery capacity over time so as to minimise the expected discounted
imbalance cost.
\paragraph{State and action spaces.}
In this benchmark, the committed generation is fixed in advance and is not a decision variable of the battery controller. Consequently, the price $p_t$ and mismatch are unaffected by the battery action. We therefore group them into the uncontrolled state $z_t=(p_t,D_t)$.
\begin{equation*}
        e_t=10\,b_p(p_t)+b_D(D_t)\in\{0,\ldots,99\}.
\end{equation*}
The battery state of charge belongs to $\mathcal X=\{0,0.5,\ldots,10\},$ and the action set is $\mathcal A=\{-2,-1.5,\ldots,2\}\ \text{kWh}. $ Consequently, the complete MDP contains $100\times21=2{,}100$ states and nine actions. Given a battery level $x$ and nominal action $a$, the feasible action and the next battery state are
\begin{equation}
        \widetilde a(x,a)  =\operatorname{clip}(a,-x,10-x), \qquad  x^+=x+\widetilde a(x,a).
\end{equation}

Because the battery grid and action increments both have spacing $0.5$ kWh, the next battery state always belongs to $\mathcal X$ and no additional projection error is introduced. The realised one-step cost is
\begin{equation*}
    c(p,D,x,a) =  p\max\{0,-(D+\widetilde a(x,a))\}.
    \label{eq:experimental-wind-cost}
\end{equation*}
All planning recursions use discount factor $\gamma=0.95$ and are implemented directly in cost-minimisation form. Performance is evaluated using the undiscounted cumulative realised cost in the original units of
\eqref{eq:experimental-wind-cost}.

\paragraph{Evaluation trajectory.}
As in the BOLA experiment, policies are evaluated for $T=3{,}000$ time steps, using indices $t=1952,\ldots,4951,$ and an initial battery level of $x_{1952}=5$ kWh. All methods start from the same battery level and face the same price and wind sequence. However, their states of charge may evolve differently because they choose different charging and discharging actions.

\subsection{Forecast}
\label{app:exp-forecast-model}

At each time step $u$, the forecast is generated as
\begin{equation}
\widehat p_u=p_u(1+\sigma\varepsilon^p_u), \qquad \widehat D_u=D_u(1+\sigma\varepsilon^D_u), \qquad \varepsilon^p_u,\varepsilon^D_u \overset{\mathrm{i.i.d.}}{\sim}\mathcal N(0,1).
    \label{eq:experimental-forecast-noise}
\end{equation}
The case $\sigma=0$ corresponds to perfect predictions. We consider $    \sigma\in
    \{0,.05,.10,.20,.30,.40,.50,.60,.80,1\}.$

Following the convention of \citet{lu2025reinforcementlearningimperfecttransition}, figures report the relative noise level as $100\sigma\%$. For each seed, we first generate one sequence of noise variables $\{(\varepsilon^p_u,\varepsilon^D_u)\}_u$, with one pair for each date $u$. This same sequence is used by every method and every look-ahead depth. For a noise level $\sigma$, the forecast at date $u$ is obtained by multiplying these fixed noise variables by $\sigma$ as in \eqref{eq:experimental-forecast-noise}. Hence, two methods evaluated with the same seed and $\sigma$ receive exactly the same forecast for any given date. Differences in performance therefore cannot be attributed to different random forecast errors.

\paragraph{Information convention.}
Throughout the experiments, $\ell$ denotes the number of exogenous observations available before acting. At date $t$, every method observes $(p_t,D_t)$ exactly and receives predictions for $t+1,\ldots,t+\ell.$ In particular, $\ell=0$ is the classical MDP without look-ahead. The cost of the current action is always evaluated using the exact current price and mismatch, $    c_t(x,a) =  p_t\max\{0,-(D_t+\widetilde a(x,a))\}.$ 

\paragraph{Posterior kernels.}
RPTAS and \textsc{Bola-Bayes} convert each noisy forecast into a posterior
transition kernel. To approximate the continuous distribution within each
bin, we sample without replacement $4{,}000$ historical pairs
$(p_i,D_i)$ from the training interval, using seed $7$. These pairs,
referred to as atoms, retain both their exogenous bin $e_i$ and their exact
cost under each feasible battery action.

\subsection{Exogenous model and posterior kernels}
\label{app:exp-exogenous-model}

The exogenous transition model is estimated on a historical interval that starts $1{,}000$ observations after the end of the evaluation trajectory, so the training and evaluation intervals are disjoint. Price and mismatch transitions are estimated separately. For $q\in\{p,D\}$, we use
\begin{equation}
    \widehat P_q(j\mid i) = \frac{ 1+ \sum_{t\in\mathcal I_{\mathrm{train}}} \mathbf 1\{b_q(q_t)=i,\ b_q(q_{t+1})=j\} }{ 10+ \sum_{t\in\mathcal I_{\mathrm{train}}} \mathbf 1\{b_q(q_t)=i\} },
    \label{eq:experimental-marginal-kernel}
\end{equation}
where the additive terms correspond to one pseudo-count for each possible successor bin. The joint transition kernel is then
\begin{equation}
    \widehat P_{\mathrm{exo}} \bigl((i',j')\mid(i,j)\bigr) =  \widehat P_p(i'\mid i)\widehat P_D(j'\mid j).
    \label{eq:experimental-exogenous-kernel}
\end{equation}
To convert a noisy forecast into a posterior over the next exogenous state, we sample without replacement $4{,}000$ historical pairs $(p_r,D_r)$ from the training interval. Each pair, called an atom, is associated with its discrete bin $e_r$ and its exact cost under every battery action. Given a forecast $\widehat z_u=(\widehat p_u,\widehat D_u)$, the likelihood weight of atom $r$ is
\begin{equation}
    w_{u,r} = \frac{1}{s^p_r s^D_r} \exp\left( -\frac{(\widehat p_u-p_r)^2}{2(s^p_r)^2}  -\frac{(\widehat D_u-D_r)^2}{2(s^D_r)^2} \right),
    \label{eq:experimental-atom-weight}
\end{equation}
where $    s^p_r=\max\{\sigma|p_r|,10^{-3}\Delta_p\}, \ s^D_r=\max\{\sigma|D_r|,10^{-3}\Delta_D\},$ and $\Delta_p,\Delta_D$ are the corresponding bin widths. The likelihood of a bin $e'$ is the average weight of its atoms:
\begin{equation}
    L_u(e')
    =
    \frac{1}{n_{e'}}
    \sum_{r:e_r=e'}w_{u,r},
    \label{eq:experimental-bin-likelihood}
\end{equation}
where $n_{e'}$ is the number of atoms in that bin. We use the average rather than the sum because the empirical frequency of the bin is already accounted for by the prior kernel $\widehat P_{\mathrm{exo}}$. Combining this likelihood with the transition prior gives
\begin{equation}
    \kappa_u(e,e')  =  \frac{    \widehat P_{\mathrm{exo}}(e,e')L_u(e')   }{     \sum_j\widehat P_{\mathrm{exo}}(e,j)L_u(j)   }.
    \label{eq:experimental-posterior-kernel}
\end{equation}
Future costs are estimated by averaging the exact atom costs using the same weights. For bins containing no atom, we use the cost evaluated at the bin centre. When $\sigma=0$, the future exogenous bin is observed exactly and $\kappa_u$ reduces to a point mass on that bin.

RPTAS and \textsc{Bola-Bayes} use the same estimated transition kernel, posterior kernels, and posterior costs. The value of $\sigma$ is assumed known when computing these quantities. The certainty-equivalent \textsc{Mpc} and \textsc{Bola} controllers do not construct a posterior; they plan directly with the point forecasts $(\widehat p_u,\widehat D_u)$.

\subsection{Compared methods}
\label{app:exp-methods}

We compare four controllers:

\paragraph{\textsc{Rptas}.}
For each pair $(\ell,\sigma)$, we draw a dictionary of $N=8$ posterior kernels using the model in \eqref{eq:experimental-posterior-kernel}. For every current exogenous bin $e$, a dictionary row is generated by sampling a successor bin from $\widehat P_{\mathrm{exo}}(e,\cdot)$, sampling an atom in that bin, corrupting it according to \eqref{eq:experimental-forecast-noise}, and computing the resulting posterior row. Rows are generated independently; at $\sigma=0$, they are one-hot. We then run the cost-minimisation version of Algorithm~\ref{alg:rptas-preprocessing} on the dictionary state space $\mathcal S\times[N]^\ell$. Value iteration is initialised at zero and stopped when the sup-norm difference between consecutive iterates is below $10^{-8}$, or after $250$ iterations. At $\ell=0$, we run standard value iteration directly under $\widehat P_{\mathrm{exo}}$. At time $t$, the online recursion receives $\kappa_{t+1},\ldots,\kappa_{t+\ell}$, the associated posterior costs, and the exact current cost. RPTAS selects one action, observes the next state, and repeats the computation using the shifted forecast window. The dictionary is sampled with seed $0$ and held fixed across the ten evaluation seeds.

\paragraph{\textsc{Bola-Bayes}.}
This controller implements the two-stage BOLA principle described by \citet{lu2025reinforcementlearningimperfecttransition}. Its offline stage computes a state-only Bayesian continuation value $J_{\ell,\sigma}(e,x)$. For every $(\ell,\sigma)$, the Bellman operator is approximated using $N_2=8$ synthetic prediction windows sampled from the same forecast model as the RPTAS dictionary. Its fixed-point iteration uses seed $0$, tolerance $10^{-4}$, and at most $500$ iterations. Online, \textsc{Bola-Bayes} uses the posterior kernels and posterior costs to select an action sequence for dates $t,\ldots,t+\ell$. The unobserved transition following date $t+\ell$ is integrated using $\widehat P_{\mathrm{exo}}$, after which $J_{\ell,\sigma}$ supplies the terminal cost. The complete sequence of $\ell+1$ actions is executed before a new forecast window is requested.

\paragraph{\textsc{Bola} and \textsc{Mpc}.}
The controller denoted \textsc{Bola} is the zero-terminal, certainty-equivalent block controller used in the wind-farm implementation of \citet{lu2025reinforcementlearningimperfecttransition}, aligned with our information convention and common forecast stream. It uses the exact current cost and the point forecasts for dates $t+1,\ldots,t+\ell$, computes an $\ell+1$-action plan, and executes the entire plan.

\textsc{Mpc} uses exactly the same finite-horizon recursion and zero terminal value but executes only its first action. It then replans at the next date using the shifted forecast window. Consequently, the
\textsc{Mpc}/\textsc{Bola} comparison isolates the effect of committing to
the complete block.

Our notation differs from that of the BOLA implementation: its horizon $K$ counts the current date as part of the action block, whereas $\ell$ counts only strictly future observations. The corresponding parameters are therefore related by $K=\ell+1.$

\paragraph{Scope of the comparison.}
RPTAS and \textsc{Bola-Bayes} receive the same Bayesian forecast information and are the main algorithmic comparison. However, they differ both in execution and in their representation of future value: RPTAS uses a value indexed by the dictionary window and replans at every step, whereas \textsc{Bola-Bayes} uses the state-only value $J_{\ell,\sigma}$ and commits to a block. Their comparison should therefore not be interpreted as an ablation of block commitment alone. That effect is isolated by the \textsc{Mpc}/\textsc{Bola} pair.

\subsection{Metric and evaluation}
\label{app:exp-common-protocol}

The main sweep uses
\[
    \ell\in\{0,1,2,3\},
    \qquad
    \sigma\in
    \{0,.05,.10,.20,.30,.40,.50,.60,.80,1\},
\]
and forecast seeds $100,\ldots,109$. For a fixed seed, all methods use the same realised price and wind-mismatch sequence; initial battery level and battery constraints; exact current price and mismatch; underlying noisy forecast at every absolute future date.

The methods do not necessarily make the same number of planning calls:
RPTAS and \textsc{Mpc} replan after every transition, while the BOLA
controllers request a new window only after completing their current
block. This is the intended distinction between receding-horizon control
and block commitment. Nevertheless, whenever two methods use a prediction
for the same absolute date, that prediction is generated from the same
noise variables.

Offline quantities are recomputed for every $(\ell,\sigma)$ but are held
fixed across the ten forecast seeds. Thus, the uncertainty bands in the
main experiment quantify sensitivity to forecast noise conditional on the
sampled dictionary and Bayesian terminal-value approximation; they do not
include dictionary-sampling variability.

For a controller $A$, its final cumulative realised cost is
\begin{equation}
    C_A
    =
    \sum_{t=1952}^{4951}
    c(p_t,D_t,x_t^A,a_t^A).
    \label{eq:experimental-cumulative-cost}
\end{equation}
Following \citet{lu2025reinforcementlearningimperfecttransition}, we report
the cost reduction relative to their no-prediction reference:
\begin{equation}
    \mathrm{CR}_A=C_{\mathrm{NP}}-C_A.
    \label{eq:experimental-cost-reduction}
\end{equation}
Higher values indicate better performance, and
$\mathrm{CR}_A<0$ means that the controller performs worse than the
no-prediction reference.

The reference $C_{\mathrm{NP}}$ is computed using the no-prediction policy
from the BOLA notebook. That policy performs $100$ value-iteration updates
while keeping the price and mismatch indices fixed and propagating only
the battery state. We retain it because
\eqref{eq:experimental-cost-reduction} is the reporting convention of the
benchmark. It is distinct from the correctly modelled classical MDP at
$\ell=0$, which propagates the exogenous state using
$\widehat P_{\mathrm{exo}}$. We therefore display the two quantities as
separate reference lines.

For each $(A,\ell,\sigma)$, we report the mean over the ten common forecast
seeds. Shaded regions correspond to
\[
    \text{mean}\ \pm\
    2\frac{\widehat{\operatorname{sd}}(C_A)}{\sqrt{10}}.
\]
Since $C_{\mathrm{NP}}$ is deterministic, the same bands apply to the
reported cost reductions.

\section{Extension to noisy transition look-ahead}
\label{app:noisy-lookahead}

This section extends our results to locally corrupted transition look-ahead. We first specify the observation model and the correct comparator, then give the planning algorithm and its guarantee, and finally the learning algorithm and its regret analysis.

\paragraph{Conventions.}
A tilde marks objects of the noisy model: $\widetilde\Theta_t$, $\widetilde L$, $\widetilde P$, $\widetilde{\mathcal T}_\ell$, $\widetilde V^\star_\ell$, $\widetilde Q^\star_\ell$. As in Appendices~\ref{proof:main-RPTAS} and~\ref{proof:regret}, $N$ is a dictionary size, $M$ a number of observed transitions, $\widehat{\widetilde L}_N$ the empirical law of a dictionary of $N$ noisy tables and $\widehat{\widetilde L}_M$ the product law of empirical rows built from $M$ observed noisy tables; the two are distinct objects. Computed action scores are written $\widetilde Q_v$ (Appendix~\ref{app:noisy-planning}) and $\widetilde Q_M$ (Appendix~\ref{app:noisy-learning-simple}), following the convention of Appendices~\ref{proof:main-RPTAS}--\ref{proof:regret}; the star distinguishes them from the optimal scores $\widetilde Q^\star_\ell$. We write $V_{\max}:=(1-\gamma)^{-1}$, $|\Omega|=|\mathcal S|^{|\mathcal S||\mathcal A|}$, and $\widetilde\theta_{k:k'}:=(\widetilde\theta_k,\dots,\widetilde\theta_{k'})$.

\subsection{Noisy transition look-ahead}
\label{app:noisy-model}

Recall that $\Omega=\mathcal S^{\mathcal S\times\mathcal A}$ and that the transition tables are i.i.d.\ with product law
\begin{equation}
  L(\theta) = \prod_{(s,a)\in\mathcal S\times\mathcal A} P\bigl(\theta(s,a)\mid s,a\bigr). \label{eq:noisy-latent-product-law}
\end{equation}
For every $(s,a,s')$, let $I(\cdot\mid s,a,s')$ be a distribution on $\mathcal S$. Conditional on $\Theta_t=\theta$, the noisy table $\widetilde\Theta_t$ is sampled according to
\begin{equation}
  \mathbb P\bigl(\widetilde\Theta_t=\widetilde\theta  \mid \Theta_t=\theta \bigr) = \prod_{(s,a)\in\mathcal S\times\mathcal A} I\bigl(\widetilde\theta(s,a)\mid s,a,\theta(s,a)\bigr).
  \label{eq:noisy-local-channel}
\end{equation}
The pairs $(\Theta_t,\widetilde\Theta_t)$ are independent across time. Before choosing $A_t$, the learner observes
\begin{equation}
  \widetilde C_t^\ell := (\widetilde\Theta_t,\ldots, \widetilde\Theta_{t+\ell-1}),
  \label{eq:noisy-window}
\end{equation}
whereas the true transition remains $S_{t+1}=\Theta_t(S_t,A_t)$. After acting, the agent observes only $S_{t+1}$; in particular, it never observes the complete table $\Theta_t$.

Let $\widetilde{L}$ be the marginal law of a noisy table; it is again a product distribution on $\Omega$. The true successor conditional on the noisy entry is governed by
\begin{equation}
  \widetilde P(s'\mid s,a,z) := \mathbb P\bigl( \Theta_t(s,a)=s'  \mid \widetilde\Theta_t(s,a)=z \bigr). \label{eq:noisy-posterior}
\end{equation}
Whenever the conditioning event has positive probability, Bayes' rule gives
\begin{equation}
  \widetilde P(s'\mid s,a,z) = \frac{P(s'\mid s,a)I(z\mid s,a,s')}{\sum_{s''\in\mathcal S}P(s''\mid s,a)I(z\mid s,a,s'')},
  \label{eq:noisy-bayes}
\end{equation}
and the posterior may be chosen arbitrarily on null conditioning events. The product assumptions imply the local sufficiency identity
\begin{equation}
  \mathbb P\bigl( \Theta_t(s,a)=s' \mid \widetilde\Theta_t=\widetilde\theta \bigr) = \widetilde P\bigl( s'\mid s,a,\widetilde\theta(s,a) \bigr).
  \label{eq:noisy-local-sufficiency}
\end{equation}

For $V:\mathcal S\times\Omega^\ell\to[0,V_{\max}]$, the Bellman optimality operator of the noisy augmented MDP is
\begin{equation}
  (\widetilde{\mathcal T}_\ell V) (s,\widetilde\theta_{0:\ell-1})
  := \max_{a\in\mathcal A} \Biggl\{ r(s,a) +\gamma \sum_{s'\in\mathcal S} \widetilde P\bigl( s'\mid s,a,\widetilde\theta_0(s,a) \bigr)  \mathbb E_{\widetilde\Theta\sim\widetilde{L}} \bigl[ V(s',\widetilde\theta_{1:\ell-1}, \widetilde\Theta) \bigr] \Biggr\}.
  \label{eq:noisy-bellman}
\end{equation}
It is a monotone $\gamma$-contraction mapping $[0,V_{\max}]$-valued functions to $[0,V_{\max}]$-valued functions. Its fixed point and associated action-value function are denoted $\widetilde V_\ell^\star$ and $\widetilde Q_\ell^\star$.

\subsection{Near-optimal planning}
\label{app:noisy-planning}

Suppose first that one can sample from $\widetilde{L}$ and evaluate $\widetilde P$. Draw
\begin{equation}
  \widetilde\Theta^1,\ldots,\widetilde\Theta^N \overset{\mathrm{i.i.d.}}{\sim} \widetilde{L},
  \qquad
  \widehat{\widetilde L}_N:=\frac1N\sum_{j=1}^N\delta_{\widetilde\Theta^j}.
  \label{eq:noisy-planning-dictionary}
\end{equation}
On $\mathcal{D}_N=\mathcal S\times[N]^\ell$, define the proxy operator, acting on arrays $v:\mathcal D_N\to\mathbb R$,
\begin{equation}
  (\widetilde{\mathcal T}_Nv)(s,i_0,\ldots,i_{\ell-1})
  := \max_{a\in\mathcal A} \Biggl\{ r(s,a)  +\gamma\sum_{s'\in\mathcal S}  \widetilde P\bigl(  s'\mid s,a,\widetilde\Theta^{i_0}(s,a)  \bigr)  \frac1N\sum_{j=1}^N  v(s',i_1,\ldots,i_{\ell-1},j)  \Biggr\},
  \label{eq:noisy-planning-operator}
\end{equation}
let $\widetilde V^\star_N$ be its fixed point, and run $\widetilde V^0:=0$, $\widetilde V^{k+1}:=\widetilde{\mathcal T}_N\widetilde V^k$ for $k=0,\dots,K-1$.

For an observed window $\widetilde c=(\widetilde\theta_0,\ldots, \widetilde\theta_{\ell-1})\in\Omega^\ell$ and an array $v:\mathcal D_N\to\mathbb R$, define $\widetilde U^v_{m,\widetilde c}:\mathcal S\times[N]^m\to\mathbb R$ by
\begin{equation}
  \widetilde U_{\ell,\widetilde c}^{v}  (u,i_0,\ldots,i_{\ell-1})  :=v(u,i_0,\ldots,i_{\ell-1}),
  \label{eq:noisy-planning-extension-terminal}
\end{equation}
and, for $m=\ell-1,\ldots,1$,
\begin{equation}
  \widetilde U_{m,\widetilde c}^{v} (u,i_0,\ldots,i_{m-1})
  :=
  \max_{a'\in\mathcal A}  \Biggl\{   r(u,a')   +\gamma\sum_{s'\in\mathcal S}  \widetilde P\bigl(    s'\mid u,a',\widetilde\theta_m(u,a')   \bigr)   \frac1N\sum_{i_m=1}^N   \widetilde U_{m+1,\widetilde c}^{v}   (s',i_0,\ldots,i_m)
  \Biggr\}.
  \label{eq:noisy-planning-extension}
\end{equation}
The action scores at the root are
\begin{equation}
  \widetilde Q_v(s,\widetilde c,a):= r(s,a)+\gamma\sum_{s'\in\mathcal S}\widetilde P\bigl(s'\mid s,a,\widetilde\theta_0(s,a)\bigr)\frac1N\sum_{i_0=1}^N\widetilde U^v_{1,\widetilde c}(s',i_0),
  \qquad a\in\mathcal A,
  \label{eq:noisy-planning-scores}
\end{equation}
and the returned policy $\pi_D$ is greedy with respect to $\widetilde Q_{\widetilde V^K}$. Relative to Algorithms~\ref{alg:rptas-preprocessing}--\ref{alg:rptas-query}, the deterministic successor $\theta(s,a)$ is replaced everywhere by the posterior average $\sum_{s'}\widetilde P(s'\mid s,a,\widetilde\theta(s,a))(\cdot)$.

\begin{theorem}[RPTAS for noisy transition look-ahead]
\label{thm:noisy-RPTAS}
Fix $\ell\geq 2$ and $\varepsilon,\delta\in(0,1)$, and let $N$ and $K$ be as in \eqref{eq:rptas-parameters}. The algorithm above computes a policy $\pi_D$ such that
\begin{equation}
\mathbb P\Bigl( \widetilde V_\ell^{\pi_D}(s,\widetilde c) \geq \widetilde V_\ell^\star(s,\widetilde c)-\varepsilon V_{\max}, \quad \forall (s,\widetilde c)\in\mathcal S\times\Omega^\ell \Bigr) \geq 1-\delta.
\label{eq:noisy-planning-guarantee}
\end{equation}
For every fixed $\ell$, computing and storing the policy, as well as selecting an action at each decision time, require time and memory polynomial in $|\mathcal S|$, $|\mathcal A|$, $\varepsilon^{-1}$, $\log(1/\delta)$, and $(1-\gamma)^{-1}$.
\end{theorem}

\begin{proof}
The proof follows that of Theorem~\ref{thm:main-RPTAS} step by step; we indicate the changes.

\emph{Ideal empirical model.} Let $\widetilde{\mathcal T}_{\ell,N}$ be the operator \eqref{eq:noisy-bellman} with $\mathbb E_{\widetilde\Theta\sim\widetilde L}$ replaced by $\mathbb E_{\widetilde\Theta\sim\widehat{\widetilde L}_N}$, and let $\widetilde V^\star_{\ell,N}$ and $\widetilde Q^\star_{\ell,N}$ be its fixed point and action-value function. Consider the class $\widetilde{\mathcal F}:=\{\widetilde\Theta\mapsto\widetilde V^\star_\ell(u,\widetilde\theta_{1:\ell-1},\widetilde\Theta):u\in\mathcal S,\ \widetilde\theta_{1:\ell-1}\in\Omega^{\ell-1}\}$, of cardinality at most $|\mathcal S||\Omega|^{\ell-1}$, whose members are deterministic and $[0,V_{\max}]$-valued. With $\eta:=\varepsilon(1-\gamma)^2V_{\max}/4$ as in \eqref{eq:eta-choice}, Lemma~\ref{lem:uniform-Hoeffding} applies verbatim: for the choice of $N$ in \eqref{eq:rptas-parameters}, with probability at least $1-\delta$ every empirical average over the dictionary of a function of $\widetilde{\mathcal F}$ is within $\eta$ of its expectation under $\widetilde L$. On this event, for every $(s,\widetilde\theta_{0:\ell-1},a)$, the bracketed terms of $\widetilde{\mathcal T}_{\ell,N}\widetilde V^\star_\ell$ and $\widetilde{\mathcal T}_\ell\widetilde V^\star_\ell$ differ by $\gamma$ times a $\widetilde P(\cdot\mid s,a,\widetilde\theta_0(s,a))$-average of such empirical-versus-true differences; averaging with respect to a probability distribution does not enlarge the error, so $\|\widetilde{\mathcal T}_{\ell,N}\widetilde V^\star_\ell-\widetilde{\mathcal T}_\ell\widetilde V^\star_\ell\|_\infty\le\gamma\eta$. The arguments of Lemmas~\ref{lemma:fixed-point-perturbation} and~\ref{lem:Q-perturbation} then give
\begin{equation}
      \|\widetilde V^\star_{\ell,N} -\widetilde V_\ell^\star\|_\infty \leq\frac{\gamma\eta}{1-\gamma}, \qquad \|\widetilde Q^\star_{\ell,N} -\widetilde Q_\ell^\star\|_\infty \leq\frac{\gamma\eta}{1-\gamma}\le\frac{\varepsilon(1-\gamma)V_{\max}}4.
\label{eq:noisy-model-error}
\end{equation}

\emph{Restriction and extension.} With the embedding $\iota$ of \eqref{eq:dictionary-embedding} built from $\widetilde\Theta^1,\dots,\widetilde\Theta^N$ and $R_NV:=V\circ\iota$, the identity $R_N\widetilde{\mathcal T}_{\ell,N}V=\widetilde{\mathcal T}_NR_NV$ holds for every $V$, by the computation in the proof of Lemma~\ref{lem:core-restriction} with the deterministic successor replaced by the posterior average; hence $\widetilde V^\star_N=R_N\widetilde V^\star_{\ell,N}$. The backward induction of Lemma~\ref{lem:extension-identity}, with the same replacement, shows that \eqref{eq:noisy-planning-extension}--\eqref{eq:noisy-planning-scores} initialised with $v=\widetilde V^\star_N$ return $\widetilde Q_{\widetilde V^\star_N}=\widetilde Q^\star_{\ell,N}$ on every window $\widetilde c\in\Omega^\ell$. Finally, Lemma~\ref{lem:extension-Lipschitz} holds unchanged because posterior averaging is a convex combination, hence $1$-Lipschitz. Consequently, with $\|\widetilde V^K-\widetilde V^\star_N\|_\infty\le\gamma^KV_{\max}\le\varepsilon(1-\gamma)V_{\max}/4$ by the choice of $K$,
\begin{equation}
      \sup_{s,\widetilde c,a} \bigl| \widetilde Q_{\widetilde V^K}(s,\widetilde c,a)  -\widetilde Q_\ell^\star(s,\widetilde c,a) \bigr|  \leq \frac{\gamma\eta}{1-\gamma} +\gamma^{\ell}\|\widetilde V^K-\widetilde V^\star_N\|_\infty\le\frac{\varepsilon(1-\gamma)V_{\max}}2.
\label{eq:noisy-score-error}
\end{equation}
The approximate-greedy argument of the proof of Theorem~\ref{thm:main-RPTAS}, applied to $\widetilde{\mathcal T}_\ell$, proves \eqref{eq:noisy-planning-guarantee}.

\emph{Cost.} Once the averages $\frac1N\sum_jv(s',i_1,\dots,i_{\ell-1},j)$ are precomputed as in \eqref{eq:Ax}, each maximisation in \eqref{eq:noisy-planning-operator} sums over $|\mathcal S|$ successors, so one sweep costs $O(|\mathcal S|(|\mathcal S||\mathcal A|+1)N^\ell)$ and the backward extension $O(\ell|\mathcal S|(|\mathcal S||\mathcal A|+1)N^\ell)$: relative to the perfect planner, the posterior sum introduces one additional factor $|\mathcal S|$. Memory is unchanged.
\end{proof}

\subsection{Learning with noisy transition look-ahead}
\label{app:noisy-learning-simple}

We use the observation model of Appendix~\ref{app:noisy-model}, with known rewards $r(s,a)\in[0,1]$. Neither $P$ nor the corruption channel $I$ needs to be known. Write $X_t=(S_t,\widetilde C_t^\ell)$ and $r_t=r(S_t,A_t)$. The comparator observes the same noisy windows as the learner, and performance is measured by
\begin{equation}
\widetilde{\operatorname{Reg}}(T) :=\sum_{t=0}^{T-1} \bigl[(1-\gamma)\widetilde V_\ell^\star(X_t)-r_t\bigr].
\label{eq:noisy-simple-regret}
\end{equation}
Unlike perfect look-ahead, a noisy table does not reveal the physical successors, so the posterior $\widetilde P$ must be estimated from visited contexts and the planner must be optimistic. Fix a horizon $T\ge2$ and a confidence level $\delta\in(0,1)$, and abbreviate
\begin{equation}
\Lambda_{\mathrm P}:=\log\frac{8|\mathcal S|^2|\mathcal A|T}{\delta},
\qquad
\Lambda_{\mathrm D}:=\log\frac{8T|\mathcal S||\Omega|^{\ell-1}}{\delta}.
\label{eq:noisy-log-factors}
\end{equation}

\paragraph{Estimators.}
At an update based on the first $M\ge1$ observed transitions and noisy tables, define for every context $(s,a,z)\in\mathcal S\times\mathcal A\times\mathcal S$ the visit count and the empirical posterior
\begin{equation}
n_M(s,a,z) :=\sum_{t=0}^{M-1} \mathbf 1\{S_t=s,\ A_t=a,\ \widetilde\Theta_t(s,a)=z\},
\end{equation}
\begin{equation}
\widehat{\widetilde P}_M(s'\mid s,a,z) :=\frac{\sum_{t=0}^{M-1} \mathbf 1\{S_t=s,\ A_t=a,\ \widetilde\Theta_t(s,a)=z,\ S_{t+1}=s'\}} {n_M(s,a,z)},
\label{eq:noisy-simple-posterior}
\end{equation}
with an arbitrary probability vector when $n_M(s,a,z)=0$. Since the noisy tables are observed in full, set
\begin{equation}
\widehat{\widetilde L}_M :=\bigotimes_{(s,a)} \Bigl(\frac1M\sum_{t=0}^{M-1}\delta_{\widetilde\Theta_t(s,a)}\Bigr),
\label{eq:noisy-simple-table-estimator}
\end{equation}
and define the exploration bonus
\begin{equation}
 b_M(s,a,z) :=\gamma V_{\max}\sqrt{\Lambda_{\mathrm P}} \left( \sqrt{\frac{|\mathcal S|}{1\vee n_M(s,a,z)}} +\sqrt{\frac{|\mathcal S|^2|\mathcal A|}{M}} \right).
\label{eq:noisy-simple-bonus}
\end{equation}

\paragraph{Optimistic dictionary planner.}
At each update, independently of previous dictionary draws conditional on the observations, sample
\begin{equation}
Z^1,\ldots,Z^{N_M}\overset{\mathrm{i.i.d.}}{\sim}\widehat{\widetilde L}_M, \qquad N_M:=\bigl\lceil 2V_{\max}^4M^2 \Lambda_{\mathrm D}\bigr\rceil,
\qquad
K_M:=\lceil V_{\max}\log(2V_{\max}M)\rceil.
\label{eq:noisy-simple-dictionary}
\end{equation}
Starting from $v^0:=0$ on $\mathcal S\times[N_M]^\ell$, perform $K_M$ iterations of the clipped optimistic update
\begin{align}
v^{k+1}(s,i_0,\ldots,i_{\ell-1}) &:=\min\Biggl\{V_{\max},\ \max_{a\in\mathcal A}\Biggl[ r(s,a)+b_M\bigl(s,a,Z^{i_0}(s,a)\bigr) \nonumber\\ &\qquad+\gamma\sum_{s'\in\mathcal S} \widehat{\widetilde P}_M\bigl(s'\mid s,a,Z^{i_0}(s,a)\bigr) \frac1{N_M}\sum_{j=1}^{N_M} v^k(s',i_1,\ldots,i_{\ell-1},j)
\Biggr]\Biggr\}.
\label{eq:noisy-simple-finite-backup}
\end{align}
For an observed window $\widetilde c$, run the backward extension \eqref{eq:noisy-planning-extension-terminal}--\eqref{eq:noisy-planning-scores} with dictionary $Z^1,\dots,Z^{N_M}$, terminal array $v^{K_M}$, posterior $\widehat{\widetilde P}_M$, stage reward $r(u,a')+b_M(u,a',\widetilde\theta_m(u,a'))$ at level $m$, and clipping at $V_{\max}$ after each maximisation. Denote by $\widetilde Q_M(x,a)$ the resulting root score, i.e.\ the expression inside the root maximisation, before clipping, with the root action fixed. The deployed policy $\pi_M$ selects an action maximising $\widetilde Q_M(x,\cdot)$ with a fixed tie-breaking rule.

\paragraph{Schedule.}
Choose any action at time $0$. Before choosing the action at time $t\geq1$, recompute the planner with $M=t$ if $t$ is a power of two, or if some count $n_t(s,a,z)$ has just reached a power of two (including $1$); otherwise keep the previous planner. Counts are updated after every observed transition even when the planner is not recomputed. For $t\ge1$, let $M_t$ denote the most recent update time; the policy $\pi_{M_t}$ is fixed between updates.

\begin{theorem}[Regret with noisy look-ahead]
\label{thm:noisy-simple-regret}
Under the local corruption model, for every fixed $\ell\geq1$, $T\geq2$, and $\delta\in(0,1)$, the preceding algorithm satisfies, with probability at least $1-\delta$,
\begin{equation}
\widetilde{\operatorname{Reg}}(T)
\leq\widetilde O_\ell\!\left(
\frac{\sqrt{|\mathcal S|^3|\mathcal A|\,T}}{1-\gamma}
+\frac{|\mathcal S|^2|\mathcal A|}{1-\gamma}
\right),
\label{eq:noisy-simple-rate}
\end{equation}
where $\widetilde O_\ell$ hides logarithmic factors in $|\mathcal S|,|\mathcal A|,T,\delta^{-1}$ and $(1-\gamma)^{-1}$, and factors depending on $\ell$. For fixed $\ell$, the update time, memory, and per-decision computation are polynomial in $|\mathcal S|,|\mathcal A|,T,(1-\gamma)^{-1}$ and $\log(1/\delta)$.
\end{theorem}

\begin{proof}
\emph{1. Uniform model confidence.}
We claim that, with probability at least $1-\delta/2$, simultaneously for all $1\leq M\leq T$ and all $(s,a,z)$,
\begin{equation}
\|\widehat{\widetilde P}_M(\cdot\mid s,a,z)
       -\widetilde P(\cdot\mid s,a,z)\|_1
\leq2\sqrt{\frac{|\mathcal S|\,\Lambda_{\mathrm P}}{1\vee n_M(s,a,z)}},
\qquad
\|\widehat{\widetilde L}_M-\widetilde L\|_1
\leq2\sqrt{\frac{|\mathcal S|^2|\mathcal A|\,\Lambda_{\mathrm P}}{M}}.
\label{eq:noisy-simple-confidence}
\end{equation}
For the first inequality, condition on the entire noisy-table sequence. By local sufficiency \eqref{eq:noisy-local-sufficiency} and independence across dates and entries, each visit to a context $(s,a,z)$ reveals a fresh successor with law $\widetilde P(\cdot\mid s,a,z)$, independent of previously revealed successors. Equivalently, successive visits to a context consume successive elements of an i.i.d.\ stream attached to that context, and this representation remains valid for adaptive actions because the action is chosen before its latent successor is revealed. For a fixed stream and a fixed prefix length $q\ge1$, let $\widehat p_q$ be the empirical distribution of the first $q$ elements and $p=\widetilde P(\cdot\mid s,a,z)$. Since $\|\widehat p_q-p\|_1=\max_{\varsigma\in\{-1,1\}^{\mathcal S}}\varsigma^\top(\widehat p_q-p)$ and each $\varsigma^\top(\widehat p_q-p)$ is an average of $q$ independent centred variables in $[-1,1]$, Hoeffding's inequality and a union bound over the $2^{|\mathcal S|}$ sign vectors give, for every $\zeta>0$,
\begin{equation*}
\mathbb P\{\|\widehat p_q-p\|_1>\zeta\}
\leq2^{|\mathcal S|}\exp(-q\zeta^2/2).
\end{equation*}
Taking $\zeta=2\sqrt{|\mathcal S|\Lambda_{\mathrm P}/q}$ and a union bound over the $|\mathcal S|^2|\mathcal A|$ streams and $q\leq T$ gives failure probability at most $\delta/4$; this covers the random prefix lengths $n_M(s,a,z)$ without conditioning on their values. The case $n_M(s,a,z)=0$ follows from $\|\widehat p_0-p\|_1\leq2\le2\sqrt{|\mathcal S|\Lambda_{\mathrm P}}$.

For the second inequality, the first $M$ noisy tables are i.i.d.\ with product law $\widetilde L$. For a single row, the method of types gives $\mathbb E\exp[M\operatorname{KL}(\widehat p_M\Vert p)]\le(M+1)^{|\mathcal S|}$, since the number of types of $M$ samples on $|\mathcal S|$ categories is at most $(M+1)^{|\mathcal S|}$ and each type has probability at most $e^{-M\operatorname{KL}}$. Rows are independent and KL is additive over products, so $\mathbb E\exp[M\operatorname{KL}(\widehat{\widetilde L}_M\Vert\widetilde L)]\le(M+1)^{|\mathcal S|^2|\mathcal A|}$, and Markov's inequality yields, for every $w>0$,
\begin{equation}
    \mathbb P\bigl\{ M\operatorname{KL}(\widehat{\widetilde L}_M\Vert\widetilde L)>w \bigr\} \leq(M+1)^{|\mathcal S|^2|\mathcal A|}e^{-w}.
\label{eq:noisy-kl-types}
\end{equation}
Take $w=2|\mathcal S|^2|\mathcal A|\Lambda_{\mathrm P}$. Since $\log(M+1)\le\Lambda_{\mathrm P}$, the right-hand side is at most $(\delta/(8|\mathcal S|^2|\mathcal A|T))^{|\mathcal S|^2|\mathcal A|}\le\delta/(4T)$, and a union bound over $M\le T$ gives failure probability at most $\delta/4$. Pinsker's inequality, $\|\widehat{\widetilde L}_M-\widetilde L\|_1\le\sqrt{2\operatorname{KL}}\le\sqrt{2w/M}$, proves the second inequality of \eqref{eq:noisy-simple-confidence}.

\emph{2. Optimism and finite-planner accuracy.}
For the analysis, let $\overline{\mathcal T}_M$ be the full-space version of \eqref{eq:noisy-simple-finite-backup}, acting on $V:\mathcal S\times\Omega^\ell\to[0,V_{\max}]$ with the dictionary average replaced by $\mathbb E_{\widetilde\Theta\sim\widehat{\widetilde L}_M}$, let $\overline V_M$ be its fixed point and $\overline Q_M$ its unclipped action scores. This operator is monotone, a $\gamma$-contraction, and maps $[0,V_{\max}]$-valued functions into $[0,V_{\max}]$-valued functions. For any $V$ with values in $[0,V_{\max}]$ and any context $(s,a,z)$, adding and subtracting the continuation computed with the true posterior and the empirical table law, and using $|(p-p')^\top f|\le\frac12\|p-p'\|_1\operatorname{span}(f)$, the discrepancy between the empirical and true discounted continuation values is at most
\begin{equation*}
\frac{\gamma V_{\max}}2
\Bigl(
\|\widehat{\widetilde P}_M(\cdot\mid s,a,z)
-\widetilde P(\cdot\mid s,a,z)\|_1
+\|\widehat{\widetilde L}_M-\widetilde L\|_1
\Bigr)
\leq b_M(s,a,z)
\end{equation*}
on the event \eqref{eq:noisy-simple-confidence}, uniformly in $V$, including data-dependent $V$. Consequently $\overline{\mathcal T}_M\widetilde V^\star_\ell\ge\widetilde V^\star_\ell$, and monotonicity together with contraction imply
\begin{equation}
0\leq\widetilde V_\ell^\star\leq\overline V_M\leq V_{\max}.
\label{eq:noisy-simple-optimism}
\end{equation}

Conditionally on the history before an update, $\overline V_M$ is fixed and the dictionary samples $Z^1,\dots,Z^{N_M}$ are i.i.d.\ from $\widehat{\widetilde L}_M$. Hoeffding's inequality for each of the at most $|\mathcal S||\Omega|^{\ell-1}$ functions $Z\mapsto\overline V_M(s',\widetilde\theta_{1:\ell-1},Z)$, and a union bound, give
\begin{equation}
\sup_{s'\in\mathcal S,\ \widetilde\theta_{1:\ell-1}\in\Omega^{\ell-1}}
\Bigl|
\frac1{N_M}\sum_{j=1}^{N_M}\overline V_M(s',\widetilde\theta_{1:\ell-1},Z^j)
-\mathbb E_{\widetilde\Theta\sim\widehat{\widetilde L}_M}
\bigl[\overline V_M(s',\widetilde\theta_{1:\ell-1},\widetilde\Theta)\bigr]
\Bigr|
\leq\frac1{2V_{\max}M}
\label{eq:noisy-simple-dictionary-event}
\end{equation}
except with conditional probability at most $\delta/(4T)$, by the choice of $N_M$ in \eqref{eq:noisy-simple-dictionary}. A union bound over the at most $T$ possible update times makes \eqref{eq:noisy-simple-dictionary-event} hold at every update with failure probability at most $\delta/4$.

Posterior averaging and clipping are $1$-Lipschitz, so \eqref{eq:noisy-simple-dictionary-event} bounds the difference between the full-space operator $\overline{\mathcal T}_M$ and the dictionary operator \eqref{eq:noisy-simple-finite-backup}, both evaluated at $\overline V_M$, by $\gamma/(2V_{\max}M)$; contraction bounds the difference between their fixed points by $\gamma/(2M)$, and the unclipped action scores then differ by at most $\gamma\bigl(\gamma/(2M)+1/(2V_{\max}M)\bigr)=\gamma/(2M)$. The array $v^{K_M}$ approximates the restriction of the dictionary fixed point with error at most $\gamma^{K_M}V_{\max}$; after $\ell$ transitions every table of a queried window belongs to the dictionary, so the backward extension recovers the exact dictionary scores from the exact array (Lemma~\ref{lem:extension-identity}, with posterior averaging, bonus and clipping, all $1$-Lipschitz), and with $v^{K_M}$ the additional score error is at most $\gamma^{K_M+\ell}V_{\max}\leq1/(2M)$ by the choice of $K_M$. Thus $\|\widetilde Q_M-\overline Q_M\|_\infty\le1/M$, and since $\pi_M$ maximises $\widetilde Q_M$,
\begin{equation}
\overline V_M(x)
\leq\overline Q_M(x,\pi_M(x))+\frac2M
\qquad\text{for every augmented state }x.
\label{eq:noisy-simple-greedy}
\end{equation}
This remains true when the maximal score exceeds $V_{\max}$, since $\overline V_M(x)=\min\{V_{\max},\max_a\overline Q_M(x,a)\}$.

\emph{3. Regret decomposition.}
Let $\mathcal F_t$ contain the observations and algorithmic randomisation available after choosing $A_t$, before observing $S_{t+1}$ and the new noisy table $\widetilde\Theta_{t+\ell}$. Write $z_t:=\widetilde\Theta_t(S_t,A_t)$. By \eqref{eq:noisy-simple-greedy} and the uniform discrepancy bound of step~2 applied to $\overline V_{M_t}$,
\begin{equation*}
\overline V_{M_t}(X_t)-r_t
-\gamma\,\mathbb E[\overline V_{M_t}(X_{t+1})\mid\mathcal F_t]
\leq
2b_{M_t}(S_t,A_t,z_t)+\frac2{M_t}.
\end{equation*}
Together with \eqref{eq:noisy-simple-optimism}, this yields
\begin{align}
(1-\gamma)\widetilde V_\ell^\star(X_t)-r_t
&\leq2b_{M_t}(S_t,A_t,z_t)
+\frac2{M_t}
+\gamma\bigl[\overline V_{M_t}(X_{t+1})-\overline V_{M_t}(X_t)\bigr]
\\&+\gamma\bigl[
\mathbb E[\overline V_{M_t}(X_{t+1})\mid\mathcal F_t]
-\overline V_{M_t}(X_{t+1})\bigr].
\label{eq:noisy-simple-one-step-regret}
\end{align}
There are at most $J:=1+(|\mathcal S|^2|\mathcal A|+1)(1+\lfloor\log_2T\rfloor)$ epochs, including time $0$. The value differences telescope within each epoch and contribute at most $V_{\max}$ per epoch. The last bracket defines martingale differences bounded by $\gamma V_{\max}$ in absolute value; Azuma's inequality bounds their sum by $\gamma V_{\max}\sqrt{2T\log(4/\delta)}$ with failure probability at most $\delta/4$, under the original law and without conditioning on the confidence events.

\emph{4. Summation.}
The update schedule ensures $M_t\geq t/2$ and $1\vee n_{M_t}(s,a,z)\geq\tfrac12\bigl(1\vee n_t(s,a,z)\bigr)$ for all $(s,a,z)$. Grouping visits by their $|\mathcal S|^2|\mathcal A|$ possible contexts, using $\sum_{j=0}^{q-1}(1\vee j)^{-1/2}\leq2\sqrt q$ and Cauchy--Schwarz,
\begin{equation*}
\sum_{t=1}^{T-1}
\frac1{\sqrt{1\vee n_{M_t}(S_t,A_t,z_t)}}
\leq2\sqrt{2|\mathcal S|^2|\mathcal A|\,T},
\qquad
\sum_{t=1}^{T-1}\frac1{\sqrt{M_t}}
\leq2\sqrt{2T},
\qquad
\sum_{t=1}^{T-1}\frac1{M_t}\leq2(1+\log T).
\end{equation*}
Consequently,
\begin{equation*}
\sum_{t=1}^{T-1}
b_{M_t}(S_t,A_t,z_t)
\leq2\sqrt2\,\gamma V_{\max}\sqrt{\Lambda_{\mathrm P}}
\Bigl(\sqrt{|\mathcal S|^3|\mathcal A|\,T}+\sqrt{|\mathcal S|^2|\mathcal A|\,T}\Bigr).
\end{equation*}
The regret at time $0$ is at most $1$. Summing \eqref{eq:noisy-simple-one-step-regret} therefore gives, on the intersection of the confidence events,
\begin{align*}
\widetilde{\operatorname{Reg}}(T)
&\leq1+4\sqrt2\,\gamma V_{\max}\sqrt{\Lambda_{\mathrm P}}
\Bigl(\sqrt{|\mathcal S|^3|\mathcal A|\,T}+\sqrt{|\mathcal S|^2|\mathcal A|\,T}\Bigr)
+4(1+\log T)
\\& +V_{\max}J
+\gamma V_{\max}\sqrt{2T\log(4/\delta)},
\end{align*}
which implies \eqref{eq:noisy-simple-rate}. The failure probabilities are at most $\delta/2$ for the model estimates, $\delta/4$ for the dictionary approximations, and $\delta/4$ for the martingale, for a total of at most $\delta$.

Finally, $\log|\Omega|=|\mathcal S||\mathcal A|\log|\mathcal S|$, so $\Lambda_{\mathrm D}$, $N_M$ and $K_M$ are polynomial in the stated parameters for fixed $\ell$. Each dictionary has $|\mathcal S|N_M^\ell$ states, the posterior sums have $|\mathcal S|$ terms, and the backward extension has depth $\ell$; the claimed computational bounds follow as in Theorem~\ref{thm:noisy-RPTAS}.
\end{proof}

\end{document}